%% file: arxiv.tex
\documentclass{article}

    \PassOptionsToPackage{numbers, compress}{natbib}

\usepackage[final]{neurips_2021}

\usepackage[utf8]{inputenc} 
\usepackage[T1]{fontenc}    
\usepackage{hyperref}       
\usepackage{url}            
\usepackage{booktabs}       
\usepackage{amsfonts}       
\usepackage{nicefrac}       
\usepackage{microtype}      
\usepackage{xcolor}         
\usepackage{graphicx}

\usepackage{bm}
\usepackage{wrapfig} 
\usepackage{subfig}
\usepackage{float}
\usepackage{dblfloatfix}

\input{math_commands.tex}

\newcommand{\graph}[1]{\ensuremath{\mathcal{#1}}}
\newcommand{\classes}[1]{\ensuremath{\mathcal{#1}}}
\newcommand{\vertices}[1]{\ensuremath{\mathcal{#1}}}

\newcommand{\loss}[1]{\ensuremath{\mathcal{#1}}}
\newcommand{\Th}[1]{\ensuremath{#1^{\text{th}}}}
\newcommand{\func}[1]{\ensuremath{\mathsf{#1}}}

\newtheorem{theorem}{Theorem}
\newtheorem{proof}{Proof}
\newtheorem{hypothesis}{Hypothesis}

\title{SLAPS: Self-Supervision Improves Structure Learning for Graph Neural Networks}

\author{%
  Bahare Fatemi\thanks{Work was done when authors were at Borealis AI.} \\
  University of British Columbia\\
  \texttt{bfatemi@cs.ubc.ca}
   \AND
   Layla El Asri \\
   Borealis AI \\
   \texttt{layla.elasri@borealisai.com} \\
   \And
   Seyed Mehran Kazemi$^*$\\
   Google Research \\
   \texttt{mehrankazemi@google.com} \\
}

\begin{document}

\maketitle

\begin{abstract}
Graph neural networks (GNNs) work well when the graph structure is provided. However, this structure may not always be available in real-world applications.
One solution to this problem is to infer a task-specific latent structure and then apply a GNN to the inferred graph. Unfortunately, the space of possible graph structures grows super-exponentially with the number of nodes and so the task-specific supervision may be insufficient for learning both the structure and the GNN parameters.
In this work, we propose the \textbf{S}imultaneous \textbf{L}earning of \textbf{A}djacency and GNN \textbf{P}arameters with \textbf{S}elf-supervision, or SLAPS, a method that provides more supervision for inferring a graph structure through self-supervision. 
A comprehensive experimental study demonstrates that SLAPS scales to large graphs with hundreds of thousands of nodes and outperforms several models that have been proposed to learn a task-specific graph structure on established benchmarks.
\end{abstract}
\section{Introduction} \label{sec:intro}
Graph representation learning has grown rapidly and found applications in domains where a natural graph of the data points is available \citep{chami2020machine,kazemi2020relational}. Graph neural networks (GNNs) \citep{scarselli2008graph} have been a key component to the success of the research in this area. Specifically, GNNs have shown promising results for semi-supervised classification when the available graph structure exhibits a high degree of homophily (i.e. connected nodes often belong to the same class) \cite{zhu2020beyond}. 

We study the applicability of GNNs to (semi-supervised) classification problems where a graph structure is \emph{not} readily available. The existing approaches for this problem either fix a similarity graph between the nodes or learn the GNN parameters and a graph structure simultaneously (see Related Work). In both cases, one main goal is to construct or learn a graph structure with a high degree of homophily with respect to the labels to aid the GNN classification. The latter approach is sometimes called \emph{latent graph learning} and often results in higher predictive performance compared to the former approach (see, e.g., \cite{franceschi2019learning}).

We identify a supervision starvation problem in latent graph learning approaches in which the edges between pairs of nodes that are far from labeled nodes receive insufficient supervision; this results in learning poor structures away from labeled nodes and hence poor generalization. We propose a solution for this problem by adopting a multi-task learning framework in which we supplement the classification task with a self-supervised task. The self-supervised task is based on the hypothesis that a graph structure that is suitable for predicting the node features is also suitable for predicting the node labels. It works by masking some input features (or adding noise to them) and training a separate GNN aiming at updating the adjacency matrix in such a way that it can recover the masked (or noisy) features. The task is generic and can be combined with several existing latent graph learning approaches. 

We develop a latent graph learning model, dubbed SLAPS, that adopts the proposed self-supervised task. We provide a comprehensive experimental study on nine datasets (thirteen variations) of various sizes and from various domains and perform thorough analyses to show the merit of SLAPS.

Our main contributions include: 
1) identifying a supervision starvation problem for latent graph learning, 
2) proposing a solution for the identified problem through self-supervision, 
3) developing SLAPS, a latent graph learning model that adopts the self-supervised solution, 
4) providing comprehensive experimental results showing SLAPS substantially outperforms existing latent graph learning baselines from various categories on various benchmarks, and 
5) providing an implementation for latent graph learning that scales to graphs with hundreds of thousands of nodes.

\section{Related work} \label{sec:rel-work}
Existing methods that relate to this work can be grouped into the following categories. We discuss selected work from each category and refer the reader to \cite{zhu2021deep} for a full survey.

\textbf{Similarity graph:}
One approach for inferring a graph structure is to select a similarity metric and set the edge weight between two nodes to be their similarity \citep{roweis2000nonlinear,tenenbaum2000global,belkin2006manifold}. To obtain a sparse structure, one may create a kNN similarity graph, only connect pairs of nodes whose similarity surpasses some predefined threshold, or do sampling. As an example, in \cite{gidaris2019generating} a (fixed) kNN graph using the cosine similarity of the node features is created.
In \cite{DGCNN}, this idea is extended by creating a fresh graph in each layer of the GNN based on the node embedding similarities in that layer. Instead of choosing a single similarity metric, in \cite{halcrow2020grale} several (potentially weak) measures of similarity are fused.
The quality of the predictions of these methods depends heavily on the choice of the similarity metric(s).

\textbf{Fully connected graph:}
Another approach is to start with a fully connected graph and assign edge weights using the available meta-data or employ the GNN variants that provide weights for each edge via an attention mechanism \citep{velivckovic2018graph,zhang2018gaan}. This approach has been used in computer vision \citep[e.g.,][]{suhail2019mixture}, natural language processing \citep[e.g.,][]{zhu2019graph}, and few-shot learning \citep[e.g.,][]{garcia2017few}. The complexity of this approach grows rapidly making it applicable only to small-sized graphs. \citet{zhang2020graph} propose to define local neighborhoods for each node and only assume that these local neighborhoods are fully connected. Their approach relies on an initial graph structure to define the local neighborhoods.

\textbf{Latent graph learning:}
Instead of a similarity graph based on the initial features, one may use a graph generator with learnable parameters. In \cite{li2018adaptive}, a fully connected graph is created based on a bilinear similarity function with learnable parameters. In \cite{franceschi2019learning}, a Bernoulli distribution is learned for each possible edge and graph structures are created through sampling from these distributions. In \cite{yang2019topology}, the input structure is updated to increase homophily based on the labels and model predictions. In \cite{IDGL}, an iterative approach is proposed that iterates over projecting the nodes to a latent space and constructing an adjacency matrix from the latent representations multiple times.
A common approach in this category is to learn a projection of the nodes to a latent space where node similarities correspond to edge weights or edge probabilities. In
\cite{wu2018quest}, the nodes are projected to a latent space by learning weights for each of the input features. In \cite{qasim2019learning,jiang2019semi,cosmo2020latent}, a multi-layer perceptron is used for projection. In \cite{GRCN,zhao2020data}, a GNN is used for projection; it uses the node features and an initial graph structure. In \cite{kazi2020differentiable}, different graph structures are created in different layers by using separate GNN projectors, where the input to the GNN projector in a layer is the projected values and the generated graph structure from the previous layer. 
In our experiments, we compare with several approaches from this category.

\textbf{Leveraging domain knowledge:}
In some applications, one may leverage domain knowledge to guide the model toward learning specific structures. For example, in \cite{johnson2020learning}, abstract syntax trees and regular languages are leveraged in learning graph structures of Python programs that aid reasoning for downstream tasks. In \cite{jin2020graph}, the structure learning is guided for robustness to adversarial attacks through the domain knowledge that clean adjacency matrices are often sparse and low-rank and exhibit feature smoothness along the connected nodes.
Other examples in this category include \cite{jang2019brain,qasim2019learning}. In our paper, we experiment with general-purpose datasets without access to domain knowledge.

\textbf{Proposed method:} Our model falls within the latent graph learning category. 
We supplement the training with a self-supervised objective to increase the amount of supervision in learning a structure. 
Our self-supervised task is inspired by, and similar to, the pre-training strategies for GNNs \citep{hu2020strategies,hu2020gpt,jin2020self,you2020does,zhu2020self} (specifically, we adopt the multi-task learning framework of \citet{you2020does}), but it differs from this line of work as we use self-supervision for learning a graph structure whereas the above methods use it to learn better (and, in some cases, transferable) GNN parameters. 

\section{Background and notation} \label{sec:background}
We use lowercase letters to denote scalars, bold lowercase letters to denote vectors and bold uppercase letters to denote matrices. $\mI$ represents an identity matrix. For a vector $\vv$, we represent its $\Th{i}$ element as $\vv_i$. For a matrix $\mM$, we represent the $\Th{i}$ row as $\mM_i$ and the element at the $\Th{i}$ row and $\Th{j}$ column as $\mM_{ij}$. For an attributed graph, we use $n$, $m$ and $f$ to represent the number of nodes, edges, and features respectively, and denote the graph as $\graph{G}=\{ \vertices{V}, \mA, \mX\}$ where $\vertices{V}=\{v_1, \dots, v_n\}$ is a set of nodes, $\mA\in\mathbb{R}^{n\times n}$ is an adjacency matrix with $\mA_{ij}$ indicating the weight of the edge from $v_i$ to $v_j$ ($\mA_{ij}=0$ implies no edge), and $\mX\in\mathbb{R}^{n\times f}$ is a matrix whose rows correspond to node features. 

Graph convolutional networks (GCNs) \cite{kipf2017semi} are a powerful variant of GNNs. For a graph $\graph{G}=\{ \vertices{V}, \mA, \mX\}$ with a degree matrix $\mD$, layer $l$ of the GCN architecture can be defined as $\mH^{(l)} = \sigma(\hat{\mA}\mH^{(l-1)}\mW^{(l)})$ where $\hat{\mA}$ represents a normalized adjacency matrix, $\mH^{(l-1)}\in\mathbb{R}^{n\times d_{l-1}}$ represents the node representations in layer \emph{l-1} ($\mH^{(0)}=\mX$), $\mW^{(l)}\in\mathbb{R}^{d_{l-1}\times d_l}$ is a weight matrix, $\sigma$ is an activation function such as ReLU~\cite{relu}, and $\mH^{(l)}\in\mathbb{R}^{n\times d_l}$ is the updated node embeddings. For undirected graphs where the adjacency is symmetric,   $\hat{\mA}=\mD^{-\frac{1}{2}}(\mA+\mI)\mD^{-\frac{1}{2}}$ corresponds to a row-and-column normalized adjacency with self-loops, and for directed graphs where the adjacency is not necessarily symmetric, $\hat{\mA}=\mD^{-1}(\mA+\mI)$ corresponds to a row normalized adjacency matrix with self-loops. Here, $\mD$ is a (diagonal) degree matrix for $(\mA+\mI)$ defined as $\mD_{ii}=1+\sum_j \mA_{ij}$.

\section{Proposed method: SLAPS}
SLAPS consists of four components: 1) generator, 2) adjacency processor, 3) classifier, and 4) self-supervision. 
Figure~\ref{fig:slaps} illustrates these components. In the next three subsections, we explain the first three components. Then, we point out a supervision starvation problem for a model based only on these components. Then we describe the self-supervision component as a solution to the supervision starvation problem and the full SLAPS model.

\subsection{Generator} \label{sec:generator}
The generator is a function $\func{G}: \mathbb{R}^{n\times f} \rightarrow \mathbb{R}^{n\times n}$ with parameters $\bm{\theta}_\func{G}$ which takes the node features $\mX \in\mathbb{R}^{n\times f}$ as input and produces a 
matrix $\tilde{\mA}\in\mathbb{R}^{n\times n}$ as output. We consider the following two generators and leave experimenting with more sophisticated graph generators (e.g., \citep{you2018graphrnn,liu2019graph,liao2019efficient}) and models with tractable adjacency computations (e.g., \citep{choromanski2020rethinking}) as future work.

\begin{figure*}[t]
   \centering
   \includegraphics[width=0.95\textwidth]{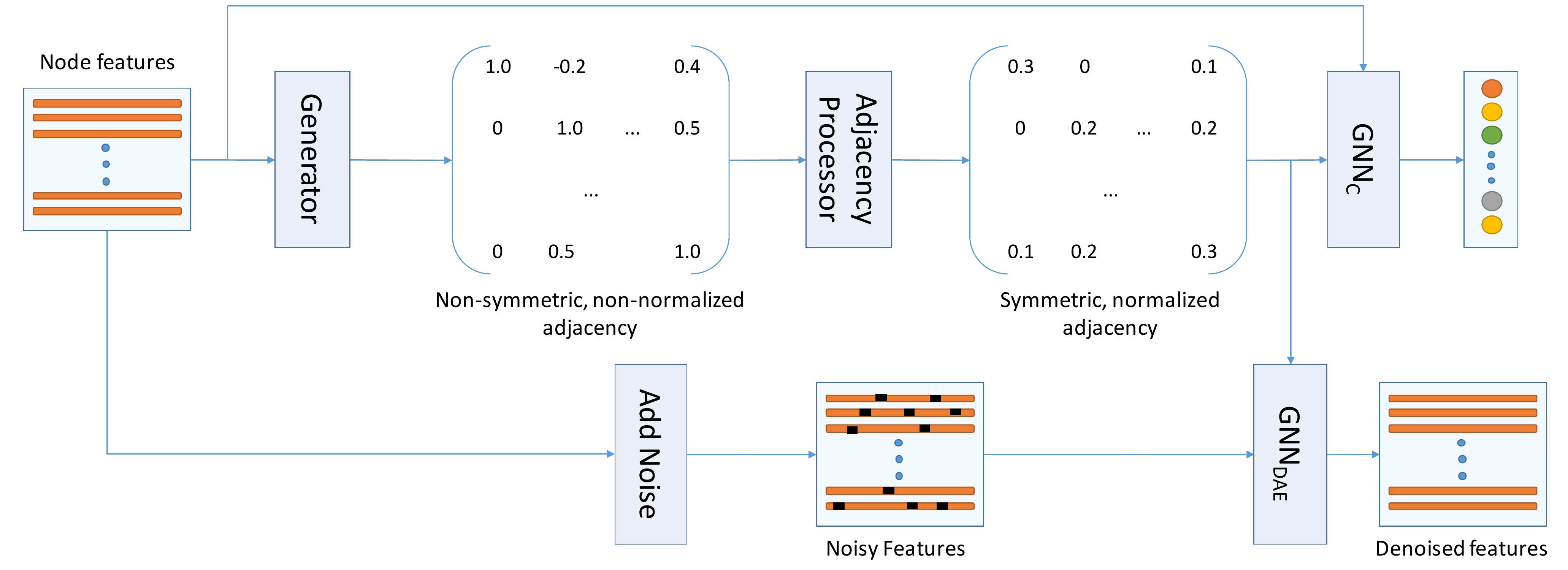}
   \caption{%
   \label{fig:slaps} %
    Overview of SLAPS. At the top, a generator receives the node features and produces a non-symmetric, non-normalized adjacency having (possibly) both positive and negative values (Section~\ref{sec:generator}). The adjacency processor makes the values positive, symmetrizes and normalizes the adjacency (Section~\ref{sec:processor}). The resulting adjacency and the node features go into $\func{GNN_C}$ which predicts the node classes (Section~\ref{sec:classifier}). At the bottom, some noise is added to the node features. The resulting noisy features and the generated adjacency go into $\func{GNN_{DAE}}$ which then denoises the features (Section~\ref{sec:self-supervision}).}
\end{figure*}

\textbf{Full parameterization (FP):} For this generator, $\bm{\theta}_\func{G}\in\mathbb{R}^{n\times n}$ and the generator function is defined as $\tilde{\mA}=\func{G}_{FP}(\mX; \bm{\theta}_\func{G})=\bm{\theta}_\func{G}$. That is, the generator ignores the input node features and directly optimizes the adjacency matrix. 
FP is similar to the generator in \cite{franceschi2019learning} except that they treat each element of $\tilde{\mA}$ as the parameter of a Bernoulli distribution and sample graph structures from these distributions. FP is simple and flexible for learning any adjacency matrix but adds $n^2$ parameters which limits scalability and makes the model susceptible to overfitting.

\textbf{MLP-kNN:} Here, $\bm{\theta}_\func{G}$ corresponds to the weights of a multi-layer perceptron (MLP) and $\tilde{\mA}=\func{G_{MLP}}(\mX; \bm{\theta}_\func{G})=\func{kNN}(\func{MLP}(\mX))$, where $\func{MLP}:\mathbb{R}^{n\times f}\rightarrow \mathbb{R}^{n\times f'}$ is an MLP that produces a matrix with updated node representations $\mX'$; $\func{kNN}:\mathbb{R}^{n\times f'}\rightarrow \mathbb{R}^{n\times n}$ produces a sparse matrix. The implementation details for the kNN operation is provided in the supplementary material. 

\textbf{Initialization and variants of MLP-kNN:}
Let $\mA^{kNN}$ represent an adjacency matrix created by applying a $\func{kNN}$ function on the initial node features. One smart initialization for $\bm{\theta}_\func{G}$ is to initialize it in a way that the generator initially generates $\mA^{kNN}$ (i.e. $\tilde{\mA}=\mA^{kNN}$ before training starts). This can be trivially done for the FP generator by initializing $\bm{\theta}_\func{G}$ to $\mA^{kNN}$. 
For MLP-kNN, we consider two variants. In one, hereafter referred to simply as MLP, we keep the input dimension the same throughout the layers. In the other, hereafter referred to as MLP-D, we consider MLPs with diagonal weight matrices (i.e., except the main diagonal, all other parameters in the weight matrices are zero). For both variants, we initialize the weight matrices in $\bm{\theta}_\func{G}$ with the identity matrix to ensure that the output of the MLP is initially the same as its input and the kNN graph created on these outputs is equivalent to $\mA^{kNN}$ (Alternatively, one may use other MLP variants but pre-train the weights to output $\mA^{kNN}$ before the main training starts.). MLP-D can be thought of as assigning different weights to different features and then computing node similarities.

\subsection{Adjacency processor} \label{sec:processor}
The output $\tilde{\mA}$ of the generator may have both positive and negative values, may be non-symmetric and non-normalized. We let $\mA = \frac{1}{2}\mD^{-\frac{1}{2}}(\func{P}(\tilde{\mA})+\func{P}(\tilde{\mA})^T)\mD^{-\frac{1}{2}}$.
Here $\func{P}$ is a function with a non-negative range applied element-wise on its input -- see supplementary material for details. The sub-expression $\frac{1}{2}(\func{P}(\tilde{\mA})+\func{P}(\tilde{\mA})^T)$ makes the resulting matrix $\func{P}(\tilde{\mA})$ symmetric. To understand the reason for taking the mean of $\func{P}(\tilde{\mA})$ and $\func{P}(\tilde{\mA})^T$, assume $\tilde{\mA}$ is generated by $\func{G_{MLP}}$. If $v_j$ is among the $k$ most similar nodes to $v_i$ and vice versa, then the strength of the connection between $v_i$ and $v_j$ will remain the same. However, if, say, $v_j$ is among the $k$ most similar nodes to $v_i$ but $v_i$ is not among the top k for $v_j$, then taking the average of the similarities reduces the strength of the connection between $v_i$ and $v_j$. Finally, once we have a symmetric adjacency with non-negative values, we normalize $\frac{1}{2}(\func{P}(\tilde{\mA})+\func{P}(\tilde{\mA})^T)$ by computing its degree matrix $\mD$ and multiplying it from left and right to $\mD^{-\frac{1}{2}}$.

\subsection{Classifier} \label{sec:classifier}
The classifier is a function $\func{GNN_C}:\mathbb{R}^{n\times f}\times\mathbb{R}^{n\times n}\rightarrow \mathbb{R}^{n\times |\classes{C}|}$ with parameters $\bm{\theta}_\func{GNN_C}$. It takes the node features $\mX$ and the generated adjacency $\mA$ as input and provides for each node the logits for each class. $\classes{C}$ corresponds to the classes and $|\classes{C}|$ corresponds to the number of classes. We use a two-layer GCN for which $\bm{\theta}_\func{GNN_C}=\{\mW^{(1)}, \mW^{(2)}\}$ and define our classifier as $\func{GNN_C}(\mA, \mX; \bm{\theta}_\func{GNN_C})=\mA\func{ReLU}(\mA\mX\mW^{(1)})\mW^{(2)}$ but other GNN variants can be used as well (recall that $\mA$ is normalized). The training loss $\loss{L}_C$ for the classification task is computed by taking the softmax of the logits to produce a probability distribution for each node and then computing the cross-entropy loss.

\subsection{Using only the first three components leads to supervision starvation} \label{sec:3comp}
One may create a model using only the three components described so far corresponding to the top part of Figure~\ref{fig:slaps}. As we will explain here, however, this model may suffer severely from supervision starvation. The same problem also applies to many existing approaches for latent graph learning, as they can be formulated as a combination of variants of these three components.

Consider a scenario during training where two unlabeled nodes $v_i$ and $v_j$ are not directly connected to any labeled nodes according to the generated structure. Then, since a two-layer GCN makes predictions for the nodes based on their two-hop neighbors, the classification loss 
\begin{wrapfigure}{r}{0.33\columnwidth}
    \centering
    \includegraphics[width=0.33\columnwidth]{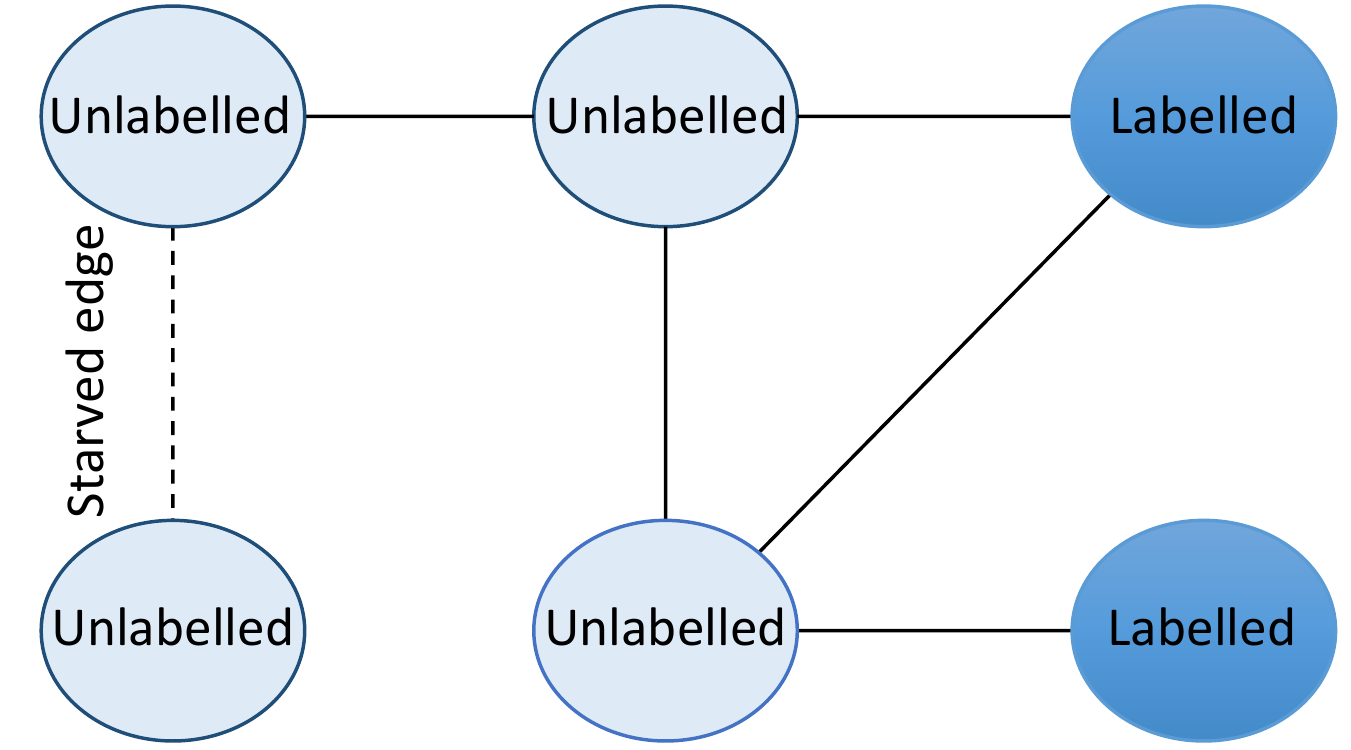}
    \caption{Using a two-layer GCN, the predictions made for the labeled nodes are not affected by the dashed (starved) edge.
    \label{fig:supervision}}
\end{wrapfigure}
(i.e. $\loss{L}_C$) is not affected by the edge between $v_i$ and $v_j$ and this edge receives no supervision\footnote{While using more layers may somewhat alleviate this problem, deeper GCNs typically produce inferior results, e.g., due to oversmoothing \citep[][]{li2018deeper,oono2020graph} -- see the supplementary material for empirical evidence.}.
Figure~\ref{fig:supervision} provides an example of such a scenario.
Let us call the edges that do not affect the loss function $\loss{L}_C$ (and consequently do not receive supervision) as \emph{starved edges}. These edges are problematic because although they may not affect the training loss, the predictions at the test time depend on these edges and if their values are learned without enough supervision, the model may make poor predictions at the test time. 
A natural question concerning the extent of the problem caused by such edges is the proportion of starved edges. The following theorem formally establishes the extent of the problem for Erd\H{o}s-R\'enyi graphs~\cite{erdos1959}; in the supplementary, we extend this result to the Barabási–Albert model \cite{albert2002statistical} and scale-free networks \cite{barabasi1999emergence}. An \emph{Erd\H{o}s-R\'enyi} graph with $n$ nodes and $m$ edges is a graph chosen uniformly at random from the collection of all graphs which have $n$ nodes and $m$ edges.

\begin{theorem} \label{thrm}
Let $\graph{G}(n, m)$ be an Erd\H{o}s-R\'enyi graph with $n$ nodes and $m$ edges. Assume we have labels for $q$ nodes selected uniformly at random. 
The probability of an edge being a starved edge with a two-layer GCN is equal to $(1 - \frac{q}{n})(1 - \frac{q}{n-1})\prod_{i=1}^{2q}(1 - \frac{m-1}{{n \choose 2}-i})$.
\end{theorem}

We defer the proof to the supplementary material. 
To put the numbers from the theorem in perspective, let us consider three established benchmarks for semi-supervised node classification namely \emph{Cora}, \emph{Citeseer}, and \emph{Pubmed} (the statistics for these datasets can be found in the Appendix). For an Erd\H{o}s-R\'enyi graph with similar statistics as the Cora dataset ($n=2708$, $m=5429$, $q=140$),
the probability of an edge being a starved edge is $59.4\%$ according to the above theorem. For Citeseer and Pubmed, this number is $75.7\%$ and $96.7\%$ respectively.
While Theorem~\ref{thrm} is stated for Erd\H{o}s-R\'enyi graphs, the identified problem also applies to natural graphs. For the original structures of Cora, Citeseer, and Pubmed, for example, $48.8\%$, $65.2\%$, and $91.6\%$ of the edges are starved edges.

\subsection{Self-supervision} \label{sec:self-supervision}
One possible solution to the supervision starvation problem is to define a \emph{prior graph structure} and regularize the learned structure toward it. 
This leads the starved edges toward the prior structure as opposed to neglecting them.
The choice of the prior is important as it determines the inductive bias incorporated into the model.
We define a prior structure based on the following hypothesis:
\begin{hypothesis}
A graph structure that is suitable for predicting the node features is also suitable for predicting the node labels.
\end{hypothesis}
We first explain why the above hypothesis is reasonable for an extreme case that is easy to understand and then extend the explanation to the general case.
Consider an extreme scenario where one of the node features is the same as the node labels. A graph structure that is suitable for predicting this feature exhibits homophily for it. Because of the equivalence between this feature and the labels, the graph structure also exhibits homophily for the labels, so it is also suitable for predicting the labels. In the general (non-extreme) case, there may not be a single feature that is equivalent to the labels but a subset of the features may be highly predictive of the labels. A graph structure that is suitable for predicting this subset exhibits homophily for the features in the subset. Because this subset is highly predictive of the labels, the structure also exhibits a high degree of homophily for the labels, so it is also suitable for predicting the node labels.

Next, we explain how to design a suitable graph structure for predicting the features and how to regularize toward it. One could design such a structure manually (e.g., by handcrafting a graph that connects nodes based on the collective homophily between their individual features) and then penalize the difference between this prior graph and the learned graph. Alternatively, in this paper, we take a learning-based approach based on self-supervision where we not only use the learned graph structure for the classification task, but also for denoising the node features. The self-supervised task encourages the model to learn a structure that is suitable for predicting the node features. We describe this approach below and provide comparisons to the manual approach in the supplementary material.

Our self-supervised task is based on denoising autoencoders \citep{vincent2008extracting}. Let $\func{GNN_{DAE}}:\mathbb{R}^{n\times f}\times \mathbb{R}^{n\times n}\rightarrow \mathbb{R}^{n\times f}$ be a GNN with parameters $\bm{\theta}_{\func{GNN_{DAE}}}$ that takes node features and a generated adjacency as input and provides updated node features with the same dimension as output. We train $\func{GNN_{DAE}}$ such that it receives a noisy version $\tilde{\mX}$ of the features $\mX$ as input and produces the denoised features $\mX$ as output. Let \emph{idx} represent the indices corresponding to the elements of $\mX$ to which we have added noise, and $\mX_{idx}$ represent the values at these indices. During training, we minimize:
\begin{equation}\label{eq-loss-dae}
    \loss{L}_{DAE}= \func{L}(\mX_{idx}, \func{\func{GNN_{DAE}}}(\tilde{\mX}, \mA; \bm{\theta}_{\func{GNN_{DAE}}})_{idx})
\end{equation}
where $\mA$ is the generated adjacency matrix and $\func{L}$ is a loss function.
For datasets where the features consist of binary vectors, $idx$ consists of $r$ percent of the indices of $\mX$ whose values are $1$ and $r\eta$ percent of the indices whose values are $0$, both selected uniformly at random in each epoch. Both $r$ and $\eta$ (corresponding to the negative ratio) are hyperparameters. In this case, we add noise by setting the $1$s in the selected mask to $0$s and $\func{L}$ is the binary cross-entropy loss. For datasets where the input features are continuous numbers, $idx$ consists of $r$ percent of the indices of $\mX$ selected uniformly at random in each epoch. We add noise by either replacing the values at $idx$ with $0$ or by adding independent Gaussian noises to each of the features. In this case, $\func{L}$ is the mean-squared error loss.

Note that the self-supervised task in \eqref{eq-loss-dae} is generic and can be added to different GNNs as well as latent graph learning models. It can be also combined with other techniques in the literature that encourage learning more homophilous structures or increase the amount of supervision. In our experiments, we test the combination of our self-supervised task with two such techniques namely \emph{self-training} \cite{li2018deeper} and \emph{AdaEdge} \cite{chen2020measuring}. Self-training helps the model ``see'' more labeled nodes and AdaEdge helps iteratively create graph structure with higher degrees of homophily. We refer the reader to the supplementary material for descriptions of self-training and AdaEdge.

\subsection{SLAPS} \label{sec:slaps}
Our final model is trained to minimize $\loss{L}=\loss{L}_C + \lambda \loss{L}_{DAE}$ where $\loss{L}_C$ is the classification loss, $\loss{L}_{DAE}$ is the denoising autoencoder loss (see \Eqref{eq-loss-dae}), and $\lambda$ is a hyperparameter controlling the relative importance of the two losses.

\section{Experiments}\label{sec:experiments}

In this section, we report our key results. More empirical comparisons, experimental analyses, and ablation studies are presented in the supplementary material.

\begin{table*}[t]
\small
\caption{Results of SLAPS and the baselines on established node classification benchmarks. $\dagger$ indicates results have been taken from \citet{franceschi2019learning}. $\ddag$ indicates results have been taken from \citet{stretcu2019graph}. Bold and underlined values indicate best and second-best mean performances respectively. \emph{OOM} indicates out of memory. \emph{OOT} indicates out of time (we allowed 24h for each run). \emph{NA} indicates not applicable.}
\setlength{\tabcolsep}{3pt}
\label{tab:results}
\begin{center}
\resizebox{\columnwidth}{!}{%
\begin{tabular}{c|cccccc}
 Model  & Cora & Citeseer & Cora390 & Citeseer370 & Pubmed 
 & ogbn-arxiv \\ \hline
MLP & $56.1 \pm 1.6^\dagger$ & $56.7 \pm 1.7^\dagger$ & $65.8 \pm 0.4$ & $67.1 \pm 0.5$ & $71.4 \pm 0.0$ & $\underline{54.7 \pm 0.1}$\\
MLP-GAM* & $70.7^\ddag$ & $70.3^\ddag$ & $-$ & $-$ & $71.9^\ddag$ & $-$\\
LP & $37.6 \pm 0.0$ & $23.2 \pm 0.0$ & $36.2 \pm 0.0$ & $29.1 \pm 0.0$ & $41.3 \pm 0.0$ & OOM\\
kNN-GCN & $66.5 \pm 0.4^\dagger$ & $68.3 \pm 1.3^\dagger$ & $72.5 \pm 0.5$ & $71.8 \pm 0.8$ & $70.4 \pm 0.4$ & $49.1 \pm 0.3$\\
LDS & $-$ & $-$ & $71.5 \pm 0.8^\dagger$ & $71.5 \pm 1.1^\dagger$ & OOM & OOM\\
GRCN & $67.4 \pm 0.3$ & $67.3 \pm 0.8$ & $71.3 \pm 0.9$ & $70.9 \pm 0.7$ & $67.3 \pm 0.3$ & OOM\\
DGCNN & $56.5 \pm 1.2$ & $55.1 \pm 1.4$& $67.3 \pm 0.7$& $66.6 \pm 0.8$ & $70.1 \pm 1.3$ & OOM\\
IDGL & $70.9 \pm 0.6$ & $68.2 \pm 0.6$ & $73.4 \pm 0.5$ & $72.7 \pm 0.4$ & $72.3 \pm 0.4$ & OOM\\
kNN-GCN + AdaEdge & $67.7 \pm 1.0$ & $68.8 \pm 1.0$ & $72.2 \pm 0.4$ & $71.8 \pm 0.6$ & OOT & OOT \\
kNN-GCN + self-training & $67.3 \pm 0.3$ & $69.8 \pm 1.0$ & $71.1 \pm 0.3$ & $72.4 \pm 0.2$ & $72.7 \pm 0.1$ & NA
\\
\hline
SLAPS (FP) & $72.4 \pm 0.4$& $70.7 \pm 0.4$& $\bm{76.6} \pm \bm{0.4}$ & $73.1\pm 0.6$ & OOM & OOM\\
SLAPS (MLP) & $72.8 \pm 0.8$ & $70.5 \pm 1.1$& $75.3 \pm 1.0$ & $73.0 \pm 0.9$ & $\bm{74.4} \pm \bm{0.6}$ & $\bm{56.6} \pm \bm{0.1}$ \\
SLAPS (MLP-D) & $\underline{73.4 \pm 0.3}$& $\underline{72.6 \pm 0.6}$& $75.1 \pm 0.5$ & $\bm{73.9} \pm \bm{0.4}$ & $73.1 \pm 0.7$ & $52.9 \pm 0.1$\\
\hline
SLAPS (MLP) + AdaEdge & $72.8 \pm 0.7$ & $70.6 \pm 1.5$ & $75.2 \pm 0.6$ & $72.6 \pm 1.4$ & OOT & OOT\\
SLAPS (MLP) + self-training & $\bm{74.2} \pm \bm{0.5}$ & $\bm{73.1} \pm \bm{1.0}$ & $\underline{75.5 \pm 0.7}$ & $\underline{73.3 \pm 0.6}$ & $\underline{74.3 \pm 1.4}$ & NA

\end{tabular}
 }
\end{center}
\end{table*}

\textbf{Baselines:} We compare our proposal to several baselines with different properties. The first baseline is a multi-layer perceptron (MLP) which does not take the graph structure into account. We also compare against MLP-GAM* \citep{stretcu2019graph} which learns a fully connected graph structure and uses this structure to supplement the loss function of the MLP toward predicting similar labels for neighboring nodes. Our third baseline is label propagation (LP)~\citep{zhu2002learning}, a well-known model for semi-supervised learning. 
Similar to \citep{franceschi2019learning}, we also consider a baseline named \emph{kNN-GCN} where we create a kNN graph based on the node feature similarities and feed this graph to a GCN; the graph structure remains fixed in this approach. We also compare with prominent existing latent graph learning models including LDS \citep{franceschi2019learning}, GRCN \citep{GRCN}, DGCNN \citep{DGCNN}, and IDGL \citep{IDGL}. In \citep{IDGL}, another variant named IDGL-ANCH is also proposed that reduces time complexity through anchor-based approximation \cite{anchor-approximation}. We compare against the base IDGL model because it does not sacrifice accuracy for time complexity, and because anchor-based approximation is model-agnostic and could be combined with other models too. We feed a kNN graph to the models requiring an initial graph structure. We also explore how adding self-training and AdaEdge impact the performance of kNN-GCN as well as SLAPS.

\textbf{Datasets:} We use three established benchmarks in the GNN literature namely Cora, Citeseer, and Pubmed \citep{sen2008collective} as well as the \emph{ogbn-arxiv} dataset \citep{hu2020open} that is orders of magnitude larger than the other three datasets and is more challenging due to the more realistic split of the data into train, validation, and test sets. For these datasets, we only feed the node features to the models and not their original graph structure. Following \cite{franceschi2019learning,IDGL}, we also experiment with several classification (non-graph) datasets available in scikit-learn~\citep{pedregosa2011scikit} including Wine, Cancer, Digits, and 20News. Furthermore, following \cite{glcn}, we also provide results on MNIST \cite{lecun2010mnist}. The dataset statistics can be found in the supplementary. For Cora and Citeseer, the LDS model uses the train data for learning the parameters of the classification GCN, half of the validation for learning the parameters of the adjacency matrix (in their bi-level optimization setup, these are considered as hyperparameters), and the other half of the validation set for early stopping and tuning the other hyperparameters. Besides experimenting with the original setups of these two datasets, we also consider a setup that is closer to that of LDS: we use the train set and half of the validation set for training and the other half of validation for early stopping and hyperparameter tuning. We name the modified versions Cora390 and Citeseer370 respectively where the number proceeding the dataset name shows the number of labels from which gradients are computed.
We follow a similar procedure for the scikit-learn datasets.

\textbf{Implementation:} We defer the implementation details and the best hyperparameter settings for our model on all the datasets to the supplementary material. Code and data is available at \href{https://github.com/BorealisAI/SLAPS-GNN}{https://github.com/BorealisAI/SLAPS-GNN}.

\subsection{Comparative results} 
The results of SLAPS and the baselines on our benchmarks are reported in Tables~\ref{tab:results}~and~\ref{tab:results2}. We start by analyzing the results in Table~\ref{tab:results} first. Starting with the baselines, we see that learning a fully connected graph in MLP-GAM* makes it outperform MLP.
kNN-GCN significantly outperforms MLP on Cora and Citeseer but underperforms on Pubmed and ogbn-arxiv. Furthermore, both self-training and AdaEdge improve the performance of kNN-GCN. This shows the importance of the similarity metric and the graph structure that is fed into GCN; a low-quality structure can harm model performance. LDS outperforms MLP but the fully parameterized adjacency matrix of LDS results in memory issues for Pubmed and ogbn-arxiv. As for GRCN, it was shown in the original paper that GRCN can revise a good initial adjacency matrix and provide a substantial boost in performance. However, as evidenced by the results, if the initial graph structure is somewhat poor, GRCN's performance becomes on par with kNN-GCN. IDGL is the best performing baseline. 

In addition to the aforementioned baselines, we also experimented with GCN, GAT, and Transformer (encoder only) architectures applied on fully connected graphs. GCN always learned to predict the majority class. This is because after one fully connected GCN layer, all nodes will have the same embedding and become indistinguishable. GAT also showed similar behavior. We believe this is because the attention weights are (almost) random at the beginning (due to random initialization of the model parameters) resulting in nodes becoming indistinguishable and GAT cannot escape from that state. The skip connections of Transformer helped avoid the problem observed for GCN and GAT and we were able to achieve better results ($\sim40\%$ accuracy on Cora). However, we observed severe overfitting (even with very small models and with high dropout probabilities).

SLAPS consistently outperforms the baselines in some cases by large margins. Among the generators, the winner is dataset-dependent with MLP-D mostly outperforming MLP on datasets with many features and MLP outperforming on datasets with small numbers of features. Using the software that was publicly released by the authors, the baselines that learn a graph structure fail on ogbn-arxiv\footnote{We note that IDGL-ANCH also scales to ogbn-arxiv.}; our implementation, on the other hand, scales to such large graphs. Adding self-training helps further improve the results of SLAPS. Adding AdaEdge, however, does not seem effective, probably because the graph structure learned by SLAPS already exhibits a high degree of homophily (see Section~\ref{sec:analysis-learned-adj}). 

\begin{table*}[t]
\scriptsize
\caption{Results on classification datasets. $\dagger$ indicates results have been taken from \citet{franceschi2019learning}. Bold and underlined values indicate best and second-best mean performances respectively.}
\label{tab:results2}
\begin{center}
\begin{tabular}{c|cccc}
 Model & Wine & Cancer & Digits & 20news \\ \hline
MLP & $96.1 \pm 1.0$ & $95.3 \pm 0.9$ & $81.9 \pm 1.0$ & $30.4 \pm 0.1$ \\
kNN-GCN & $93.5 \pm 0.7$ & $95.3 \pm 0.4$ & $\bm{95.4} \pm \bm{0.4}$ & $46.3 \pm 0.3$ \\
LDS & $\bm{97.3} \pm \bm{0.4}^\dagger$ & $94.4 \pm 1.9^\dagger$ & $92.5 \pm 0.7^\dagger$ & $46.4 \pm 1.6^\dagger$ \\
IDGL & $\underline{97.0 \pm 0.7}$ & $94.2 \pm 2.3$ & $92.5 \pm 1.3$ & $48.5 \pm 0.6$ \\
\hline
SLAPS (FP) & $96.6 \pm 0.4$ & $94.6 \pm 0.3$& $\underline{94.4 \pm 0.7}$ & $44.4 \pm 0.8$ \\
SLAPS (MLP) & $96.3 \pm 1.0$& $\underline{96.0 \pm 0.8}$& $92.5\pm0.7$ & $\bm{50.4} \pm \bm{0.7}$ \\
SLAPS (MLP-D) & $96.5 \pm 0.8$& $\bm{96.6} \pm \bm{0.2}$& $94.2 \pm 0.1$ & $\underline{49.8 \pm 0.9}$
\end{tabular}
\end{center}
\end{table*}

In Table~\ref{tab:results2}, we only compared SLAPS with the best performing baselines from Table~\ref{tab:results} (kNN-GCN, LDS and IDGL). We also included an MLP baseline for comparison. On three out of four datasets, SLAPS outperforms the LDS and IDGL baselines. For the Digits dataset, interestingly kNN-GCN outperforms the learning-based models. This could be because the initial kNN structure for this dataset is already a good structure. Among the datasets on which we can train SLAPS with the FP generator, 20news has the largest number of nodes (9,607 nodes). On this dataset, we observed that an FP generator suffers from overfitting and produces weaker results compared to other generators due to its large number of parameters.

\citet{jiang2019semi} show that learning a latent 
graph structure of the input examples can help with
\begin{wraptable}{r}{0.53\columnwidth}
\small
\caption{Results on the MNIST dataset. Bold values indicate best mean performances. Underlined values indicate second best mean performance. All the results for baseline have been taken from \cite{glcn}.}
\label{tab:mnist}
\setlength{\tabcolsep}{3pt}
\begin{tabular}{c|ccc}
 Model & MNIST1000 & MNIST2000 & MNIST3000\\ \hline
ManiReg & $92.74 \pm 0.3$ & $93.96 \pm 0.2$ & $94.62 \pm 0.2$\\
LP & $79.28 \pm 0.9$ & $81.91 \pm 0.8$ & $83.45 \pm 0.5$ \\
DeepWalk & $\underline{94.55 \pm 0.3}$ & $95.04 \pm 0.3$ & $95.34 \pm 0.3$ \\
GCN & $90.59 \pm 0.3$ & $90.91 \pm 0.2$ & $91.01\pm 0.2$ \\
GAT & $92.11 \pm 0.4$ & $92.64 \pm 0.3$ & $92.81 \pm 0.3$\\
GLCN & $94.28 \pm 0.3$ & $\underline{95.09 \pm 0.2}$ & $\underline{95.46 \pm 0.2}$\\
\hline
SLAPS & $\bm{94.66} \pm \bm{0.2}$ & $\bm{95.35} \pm \bm{0.1}$ & $\bm{95.54} \pm \bm{0.0}$\\
\end{tabular}
\end{wraptable}
semi-supervised image classification. 
In particular, they create three versions of the MNIST dataset each consisting of a randomly selected subset with 10,000 examples in total. 
The first version contains 1000 labels for training, the second contains 2000, and the third version contains 3000 labels for training. 
All three variants use an extra 1000 labels for validation. The other examples are used as test examples. 
Here, we conduct an experiment to measure the performance of SLAPS on these variants of the MNIST dataset. We compare against GLCN \cite{jiang2019semi} as well as the baselines in the GLCN paper including manifold regularization \cite{belkin2006manifold}, label propagation, deep walk \cite{perozzi2014deepwalk}, graph convolutional networks (GCN), and graph attention networks (GAT). 

The results are reported in Table~\ref{tab:mnist}. From the results, it can be viewed that SLAPS outperforms GLCN and all the other baselines on the 3 variants. Compared to GLCN, on the three variants SLAPS reduces the error by $7\%$, $5\%$, and $2\%$ respectively, showing that SLAPS can be more effective when the labeled set is small and providing more empirical evidence for Theorem 1.

\subsection{The effectiveness of self-supervision}

\begin{wrapfigure}{r}{0.35\columnwidth}
    \centering
    \includegraphics[width=0.35\columnwidth]{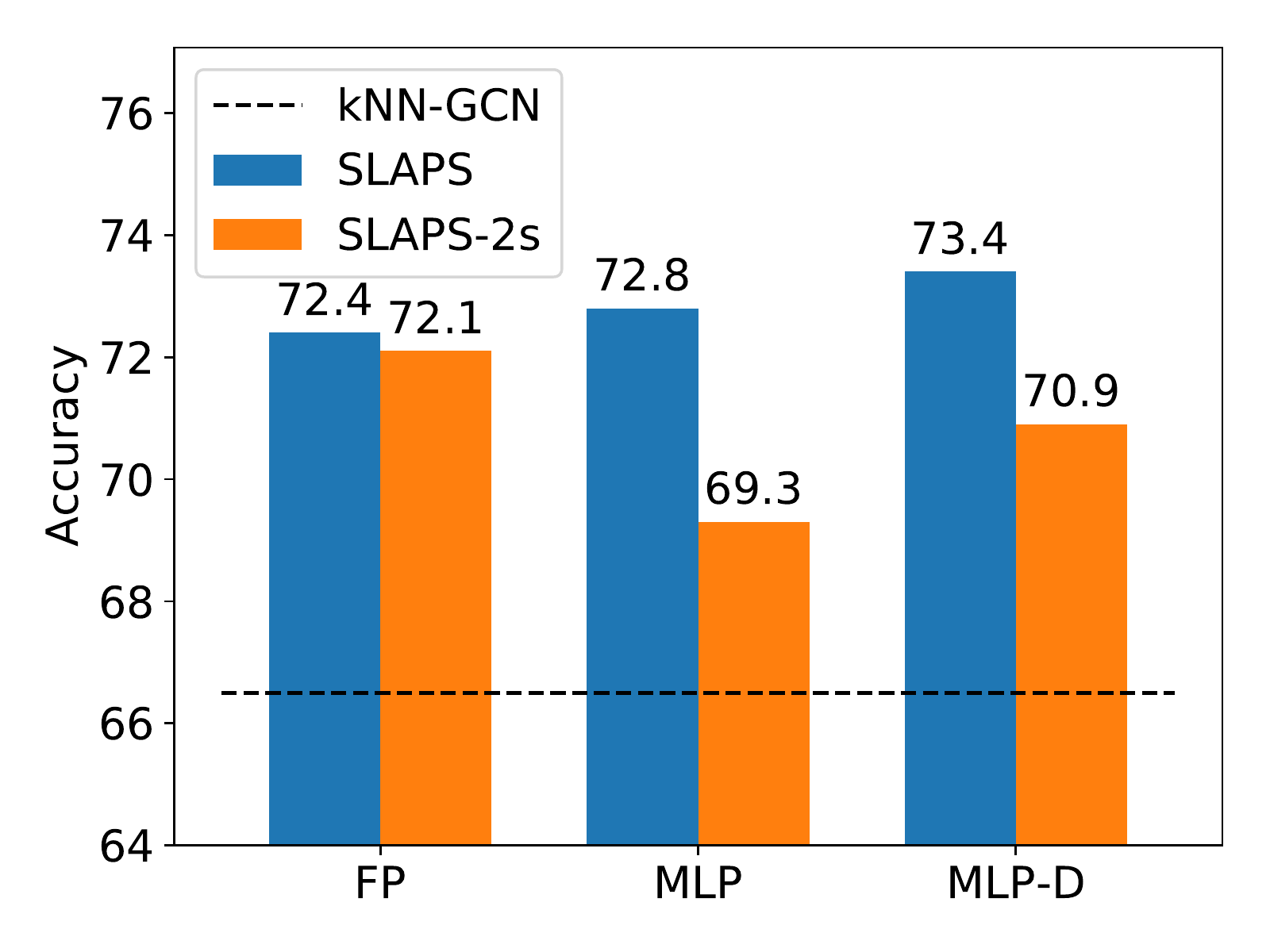}
    \caption{SLAPS vs SLAPS$_{2s}$ on Cora with different generators.}
    \label{fig:slaps_2s}
\end{wrapfigure}

\textbf{Learning a structure only using self-supervision:}
To provide more insight into the value provided by the self-supervision task and the generalizability of the adjacency learned through this task, we conduct experiments with a variant of SLAPS named $SLAPS_{2s}$ that is trained in two stages. We first train the \func{GNN_{DAE}} model by minimizing $\loss{L}_{DAE}$ described in in~\Eqref{eq-loss-dae}. Recall that $\loss{L}_{DAE}$ depends on the parameters $\bm{\theta}_\func{G}$ of the generator and the parameters $\bm{\theta}_{\func{GNN_{DAE}}}$ of the denoising autoencoder. After every $t$ epochs of training, we fix the adjacency matrix, train a classifier with the fixed adjacency matrix, and measure classification accuracy on the validation set. We select the epoch that produces the adjacency providing the best validation accuracy for the classifier. Note that in $SLAPS_{2s}$, the adjacency matrix only receives gradients from the self-supervised task in \Eqref{eq-loss-dae}.

Figure~\ref{fig:slaps_2s} shows the performance of SLAPS and SLAPS$_{2s}$ on Cora and compares them with kNN-GCN. Although SLAPS$_{2s}$ does not use the node labels in learning an adjacency matrix, it outperforms kNN-GCN ($8.4\%$ improvement when using an FP generator). With an FP generator, SLAPS$_{2s}$ even achieves competitive performance with SLAPS; this is mainly because FP does not leverage the supervision provided by $\func{GCN_C}$ toward learning generalizable patterns that can be used for nodes other than those in the training set. These results corroborate the effectiveness of the self-supervision task for learning an adjacency matrix. Besides, the results show that learning the adjacency using both self-supervision and the task-specific node labels results in higher predictive accuracy.

\begin{wrapfigure}{r}{0.35\columnwidth}
    \centering
    \includegraphics[width=0.35\columnwidth]{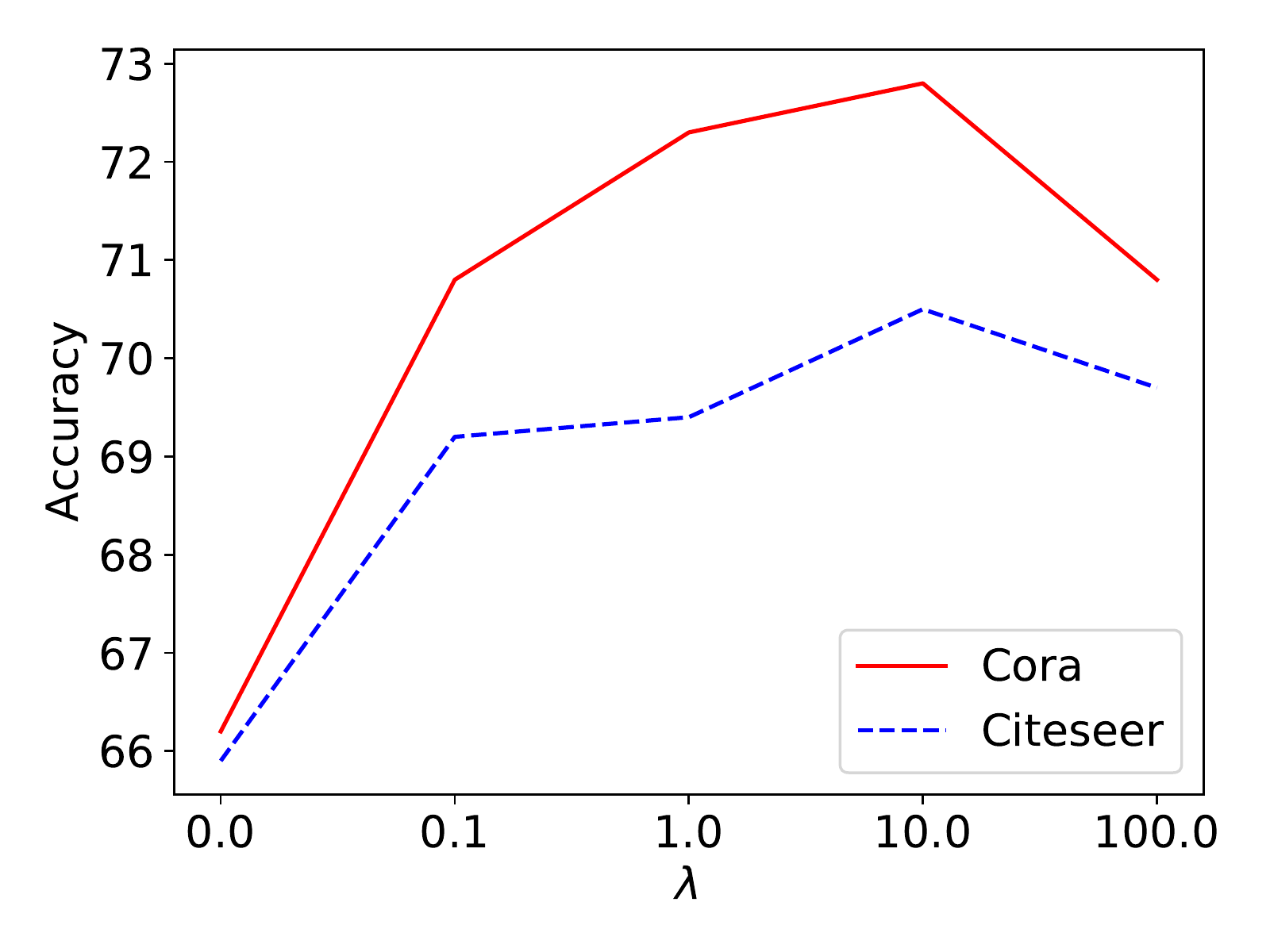}
    \caption{The performance of SLAPS with MLP graph generator as a function of $\lambda$.}
    \label{fig:lambda}
\end{wrapfigure}

\textbf{The value of $\bm{\lambda}$:} Figure~\ref{fig:lambda} shows the performance of SLAPS\footnote{The generator used in this experiment is MLP; other generators produced similar results.} on Cora and Citeseer with different values of $\lambda$. When $\lambda=0$, corresponding to removing self-supervision, the model performance is somewhat poor. As soon as $\lambda$ becomes positive, both models see a large boost in performance showing that self-supervision is crucial to the high performance of SLAPS. Increasing $\lambda$ further provides larger boosts until it becomes so large that the self-supervision loss dominates the classification loss and the performance deteriorates. Note that with $\lambda=0$, SLAPS with the MLP generator becomes a variant of the model proposed in \cite{cosmo2020latent}, but with a different similarity function.

\textbf{Is self-supervision actually solving the supervision starvation problem?} In Fig~\ref{fig:lambda}, we showed that self-supervision is key to the high performance of SLAPS. Here, we examine if this is because self-supervision indeed addresses the supervision starvation problem. For this purpose, we compared SLAPS with and without self-supervision on two groups of test nodes on Cora: 1) those that are not connected to any labeled nodes after training, and 2) those that are connected to at least one labeled node after training. The nodes in group one have a high chance of having starved edges. We observed that adding self-supervision provides $38.0\%$ improvement for the first group and only $8.9\%$ improvement for the latter. Since self-supervision mainly helps with nodes in group 1, this provides evidence that self-supervision is an effective solution to the supervision starvation problem.

\textbf{The effect of the training set size:} According to Theorem~\ref{thrm}, a smaller $q$ (corresponding to the training set size) results in more starved edges in each epoch. To explore the effect of self-supervision as a function of $q$, we compared SLAPS with and without supervision on Cora and Citeseer while reducing the number of labeled nodes per class from 20 to 5. We used the FP generator for this experiment. With 5 labeled nodes per class, adding self-supervision provides  $16.7\%$ and $22.0\%$ improvements on Cora and Citeseer respectively, which is substantially higher than the corresponding numbers when using 20 labeled nodes per class ($10.0\%$ and $7.0\%$ respectively). 
This provides empirical evidence for Theorem~\ref{thrm}. Note that the results on Cora390 and Citeseer 370 datasets provide evidence that the self-supervised task is effective even when the label rate is high.

\subsection{Experiments with noisy graphs} \label{sec:noisy}
\begin{wrapfigure}{r}{0.35\columnwidth}
    \centering
    \includegraphics[width=0.35\columnwidth]{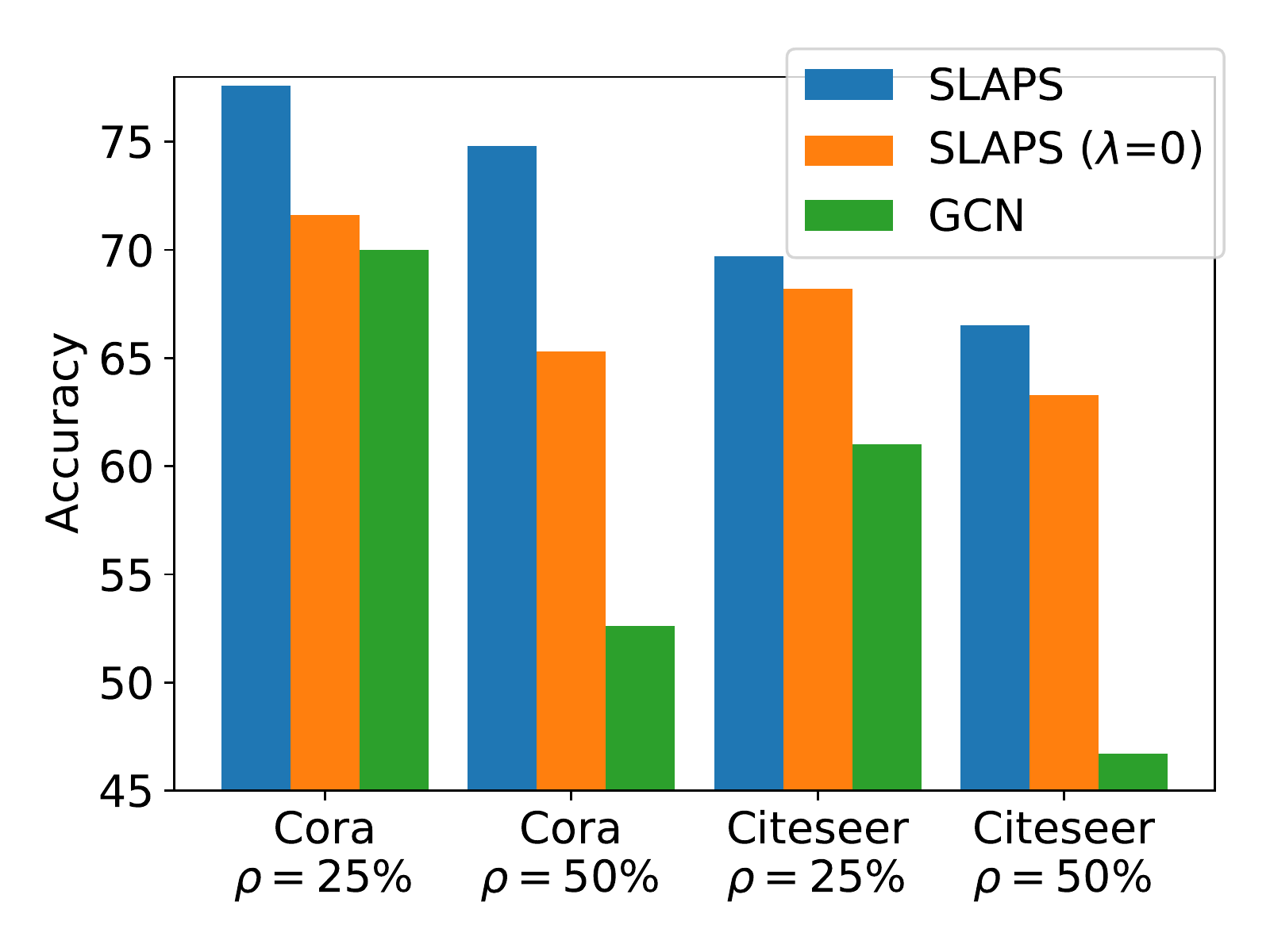}
    \caption{Performance comparison when noisy graphs are provided as input ($\rho$ indicates the percentage of perturbations).}
    \label{fig:noise}
\end{wrapfigure}
The performance of GNNs highly depends on the quality of the input graph structure and deteriorates when the graph structure is noisy \citep[see][]{zugner2018adversarial,dai2018adversarial,fox2019robust}. Here, we verify whether self-supervision is also helpful when a noisy structure is provided as input. Toward this goal, we experiment with Cora and Citeseer and provide noisy versions of the input graph as input. The provided noisy graph structure is used only for initialization; it is then further optimized by SLAPS. We perturb the graph structure by replacing $\rho$ percent of the edges in the original structure (selected uniformly at random) with random edges. Figure~\ref{fig:noise} shows the performance of SLAPS with and without self-supervision ($\lambda=0$ corresponds to no supervision). We also report the results of vanilla GCN on these perturbed graphs for comparison. It can be viewed that self-supervision consistently provides a boost in performance especially for higher values of $\rho$. 

\subsection{Analyses of the learned adjacency}
\label{sec:analysis-learned-adj}
\textbf{Noisy graphs:} 
Following the experiment in Section~\ref{sec:noisy}, we compared the learned and original structures by measuring the number of random edges added during perturbation but removed by the model and the number of edges removed during the perturbation but recovered by the model.
For Cora, SLAPS removed $76.2\%$ and $70.4\%$ of the noisy edges and recovered $58.3\%$ and $44.5\%$ of the removed edges for $\rho=25\%$ and $\rho=50\%$ respectively 
\begin{wrapfigure}{r}{0.35\columnwidth}
    \centering
    \includegraphics[width=0.35\columnwidth]{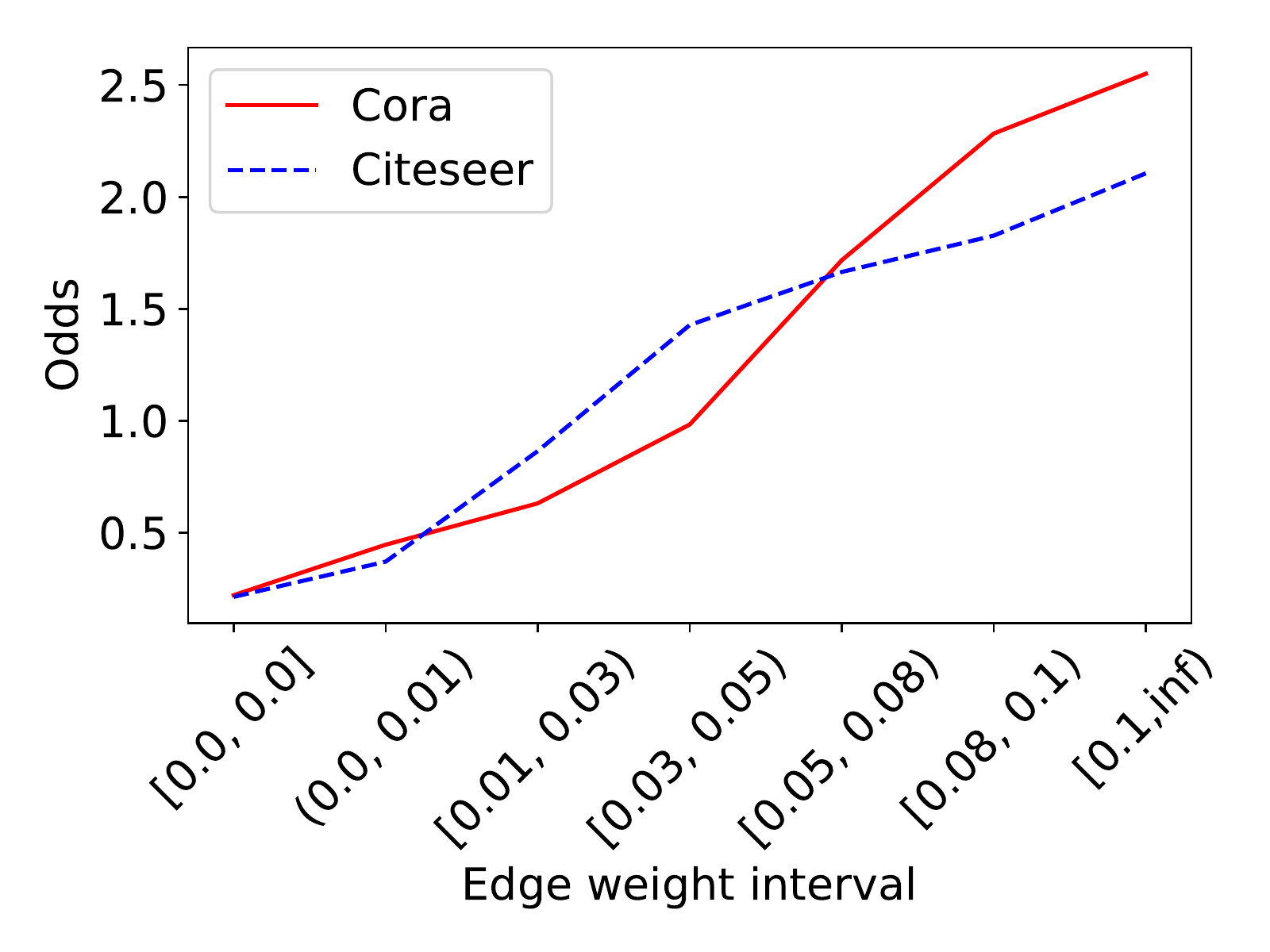}
    \caption{The odds of two nodes in the test set sharing the same label as a function of the edge weights learned by SLAPS.}
    \label{fig:odds}
\end{wrapfigure}
while SLAPS with $\lambda=0$ only removed $62.8\%$ and $54.9\%$ of the noisy edges and recovered $51.4\%$ and $35.8\%$ of the removed edges. This provides evidence on self-supervision being helpful for structure learning. 

\textbf{Homophily:} 
As explained earlier, a properly learned graph for semi-supervised classification with GNNs exhibits high homophily.
To verify the quality of the learned adjacency with respect to homophily, for every pair of nodes in the test set, we compute the odds of the two nodes sharing the same label as a function of the normalized weight of the edge connecting them. Figure~\ref{fig:odds} represents the odds for different weight intervals (recall that $\mA$ is row and column normalized). For both Cora and Citeseer, nodes' connected with higher edge weights are more likely to share the same label compared to nodes with lower or zero edge weights. Specifically, when $\mA_{ij}\geq 0.1$, $v_i$ and $v_j$ are almost 2.5 and 2.0 times more likely to share the same label on Cora and Citeseer respectively.

\section{Conclusion}
We proposed SLAPS: a model for learning the parameters of a graph neural network and a graph structure of the nodes connectivities simultaneously from data. We identified a supervision starvation problem that emerges for graph structure learning, especially when training data is scarce. We proposed a solution to the supervision starvation problem by supplementing the training objective with a well-motivated self-supervised task. We showed the effectiveness of our model through a comprehensive set of experiments and analyses.

\section{Funding Transparency Statement} 
This work was fully funded by Borealis AI.

\bibliography{MyBib.bib}
\bibliographystyle{plainnat}

\appendix

\section{More Experiments and Analyses}
\begin{wrapfigure}{r}{0.4\columnwidth}
   \includegraphics[width=0.4\textwidth]{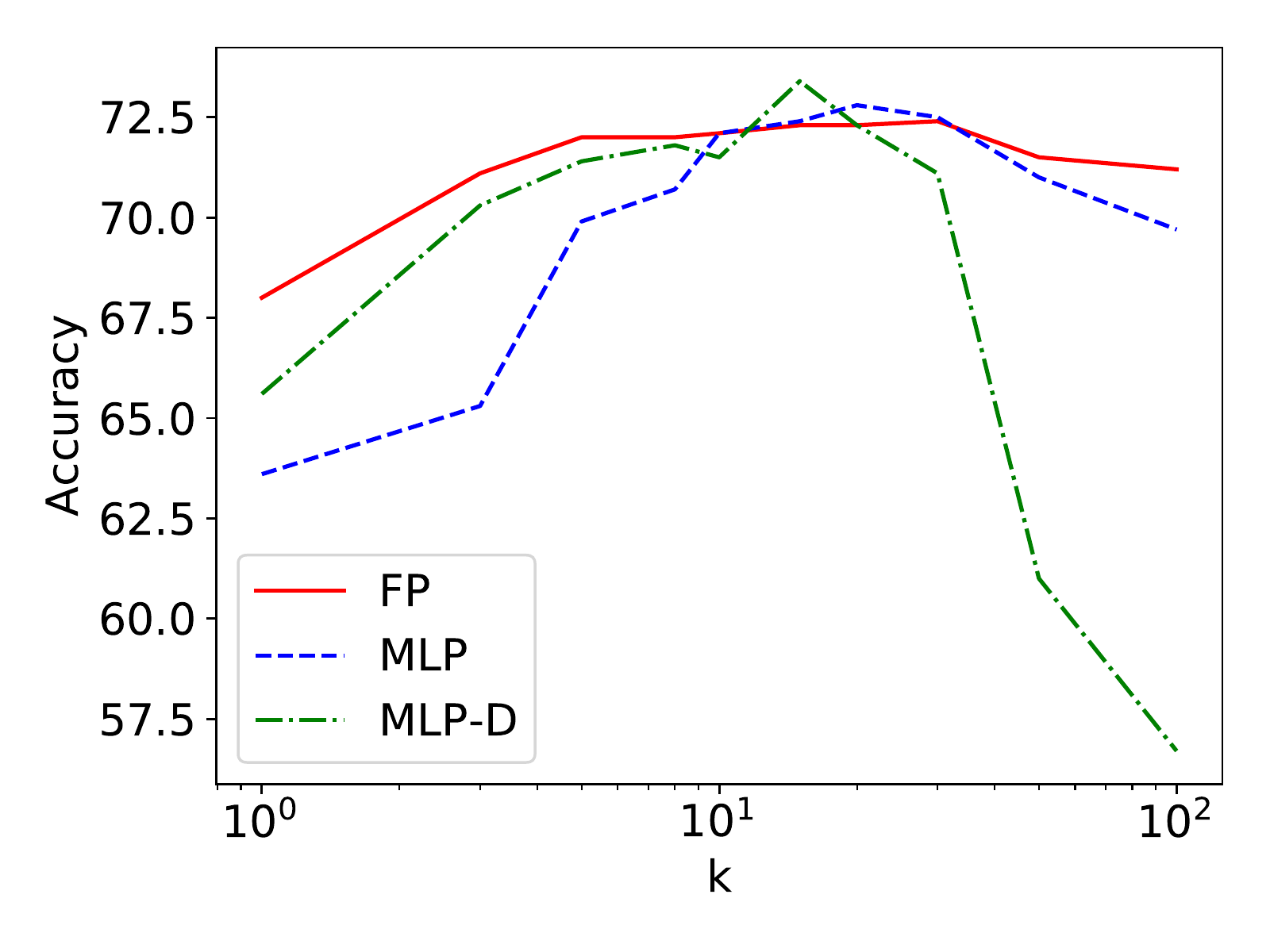}
   \caption{%
   \label{fig:knn} %
   The performance of SLAPS on Cora as a function of $k$ in kNN.}
\end{wrapfigure}
\textbf{Importance of $\bm{k}$ in kNN:} Figure~\ref{fig:knn} shows the performance of SLAPS on Cora for three graph generators as a function of $k$ in kNN. For all three cases, the value of $k$ plays a major role in model performance. The FP generator is the least sensitive because, in FP, $k$ only affects the initialization of the adjacency matrix but then the model can change the number of neighbors of each node.
For MLP and MLP-D, however, the number of neighbors of each node remains close to $k$ (but not necessarily equal as the adjacency processor can add or remove some edges) and the two generators become more sensitive to $k$. For larger values of $k$, the extra flexibility of the MLP generator enables removing some of the unwanted edges through the function $\func{P}$ or reducing the weights of the unwanted edges resulting in MLP being less sensitive to large values of $k$ compared to MLP-D.

\textbf{Increasing the number of layers:} In the main text, we described how some edges may receive no supervision during latent graph learning. We pointed out that while increasing the number of layers of the GCN may alleviate the problem to some extent, deeper GCNs typically provide inferior results due to issues such as oversmoothing \citep[see, e.g.,][]{li2018deeper,oono2020graph}. We empirically tested deeper GCNs for latent graph learning to see if simply using more layers can obviate the need for the proposed self-supervision. Specifically, we tested SLAPS without self-supervision (i.e. $\lambda=0$) with 2, 4, and 6 layers on Cora. We also added residual connections that have been shown to help train deeper GCNs \cite{li2019deepgcns}. The accuracies for 2, 4, and 6-layer models are 66.2\%, 67.1\%, and 55.8\% respectively. It can be viewed that increasing the number of layers from 2 to 4 provides an improvement. This might be because the benefit provided by a 4-layer model in terms of alleviating the starved edge problem outweighs the increase in oversmoothing. However, when the number of layers increases to 6, the oversmoothing problem outweighs and the performance drops significantly. Further increasing the number of layers resulted in even lower accuracies.

\textbf{Symmetrization:} In the adjacency processor, we used the following equation:
\begin{equation*}
    \mA = \mD^{-\frac{1}{2}}\Big(\frac{\func{P}(\tilde{\mA})+\func{P}(\tilde{\mA})^T}{2}\Big)\mD^{-\frac{1}{2}} 
\end{equation*}
\begin{wrapfigure}{r}{0.4\columnwidth}
   \includegraphics[width=0.4\textwidth]{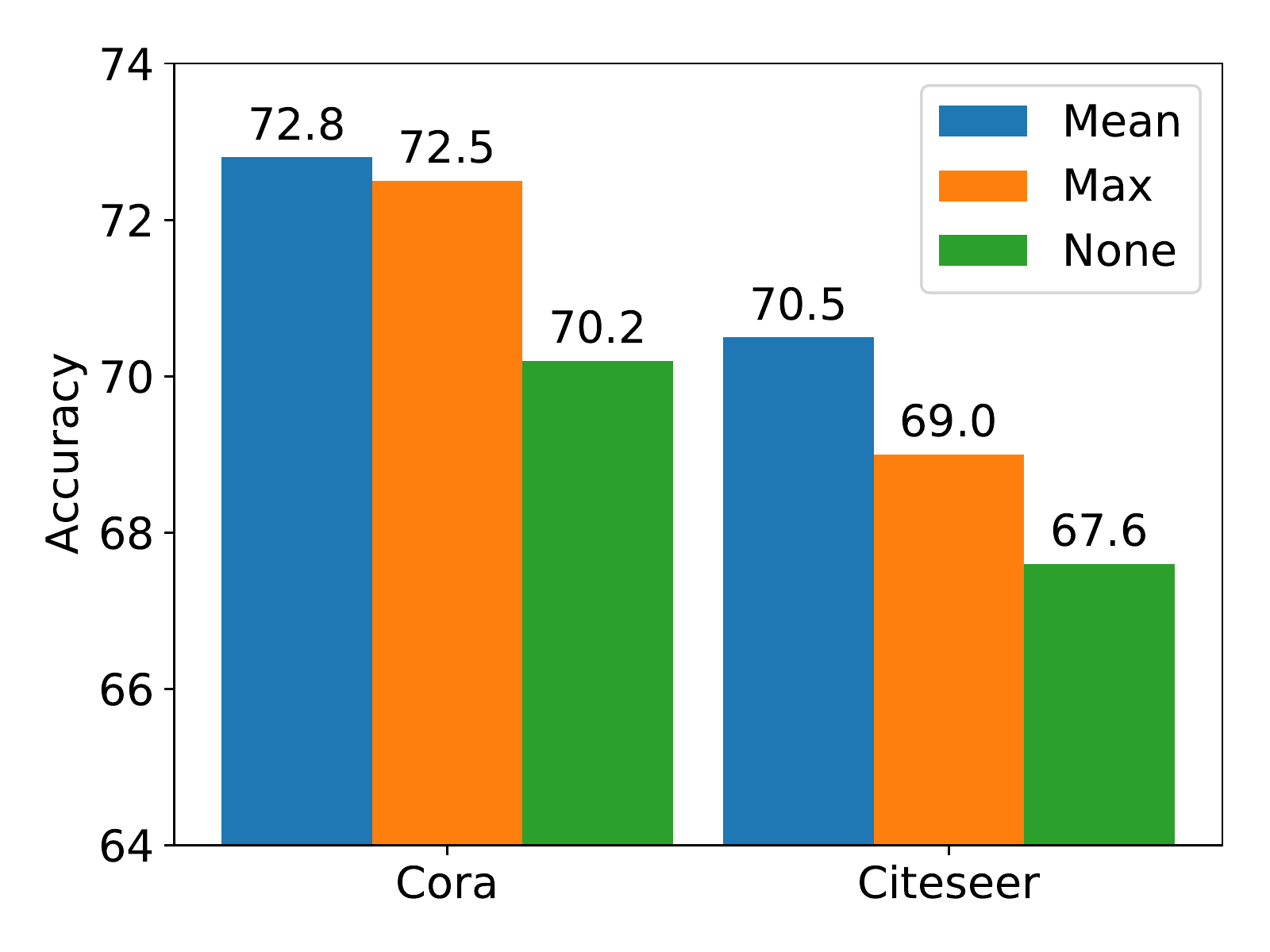}
   \caption{%
   \label{fig:symm} %
   The performance of SLAPS on Cora and Citeseer with different adjacency symmetrizations.
   }
\end{wrapfigure}
which symmetrized the adjacency matrix by taking the average of $\func{P}(\tilde{\mA})$ and $\func{P}(\tilde{\mA})^T$. Here we also consider two other choices: 1) $\func{max}(\func{P}(\tilde{\mA})$, $\func{P}(\tilde{\mA})^T$), and 2) not symmetrizing the adjacency (i.e. using $\func{P}(\tilde{\mA})$).
Figure~\ref{fig:symm} compares these three choices on Cora and Citeseer with an MLP generator (other generators produced similar results). On both datasets, symmetrizing the adjacency provides a performance boost. 
Compared to mean symmetrization, max symmetrization performs slightly worse. This may be because max symmetrization does not distinguish between the case where both $v_i$ and $v_j$ are among the $k$ most similar nodes of each other and the case where only one of them is among the $k$ most similar nodes of the other.

\textbf{Fixing a prior graph manually instead of using self-supervision:} In the main text, we validated Hypothesis 1 by adding a self-supervised task to encourage learning a graph structure that is appropriate for predicting the node features, and showing in our experiments how this additional task helps improve the results. Here, we provide more evidence for the validity of Hypothesis 1 by showing that we can obtain good results even when regularizing the learned graph structure toward a manually fixed structure that is appropriate for predicting the node features. 
\begin{wrapfigure}{r}{0.4\columnwidth}
   \includegraphics[width=0.4\textwidth]{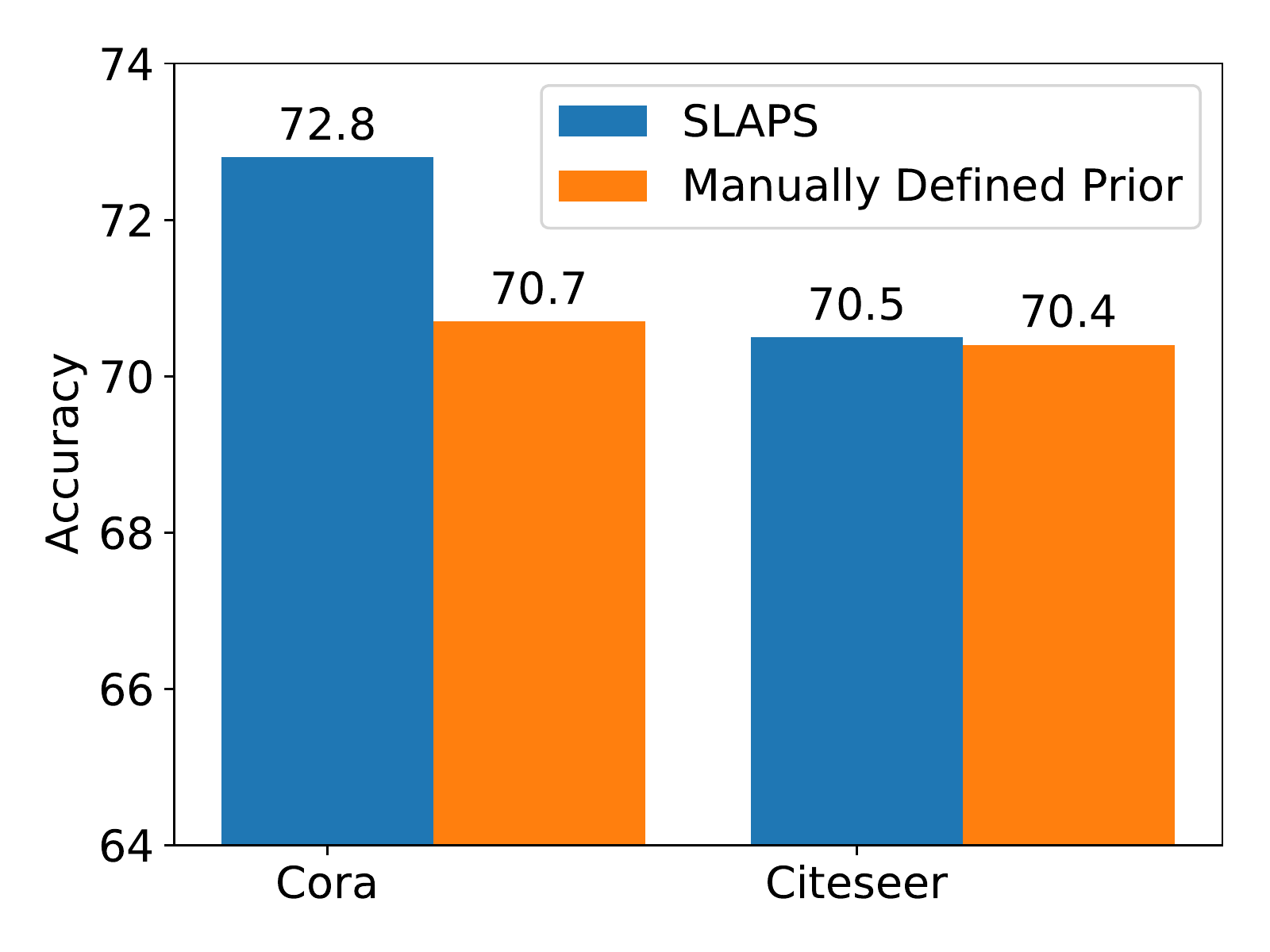}
   \caption{%
   \label{fig:slaps-vs-reg} %
   The performance of SLAPS and regularization toward a manually defined prior structure on Cora and Citeseer when using the MLP generator.
   }
\end{wrapfigure}

Toward this goal, we experimented with Cora and Citeseer and created a cosine similarity graph as our prior graph $\mA^{prior}$ where the edge weights correspond to the cosine similarity of the nodes. We sparsified $\mA^{prior}$ by connecting each node only to the $k$ most similar nodes. Then, we added a term $\lambda || \mA - \mA^{prior} ||_F$ to the loss function where $\lambda$ is a hyperparameter, $\mA$ is the learned graph structure (i.e. the output of the graph generator), and $|| . ||_F$ shows the Frobenius norm. Note that $\mA^{prior}$ exhibits homophily with respect to the node features because the node features in Cora and Citeseer are binary, so two nodes that share the same values for more features have a higher similarity and are more likely to be connected.

The results can be viewed in Figure~\ref{fig:slaps-vs-reg}. According to the results, we can see that regularizing toward a manually designed $\mA^{prior}$ also provides good results but falls short of SLAPS with self-supervision. The superiority of the self-supervised approach compared to the manual design could be due to two reasons. 
\begin{itemize}
\item Some of the node features may be redundant (e.g., they may be derived from other features) or highly correlated. These features can negatively affect the similarity computations for the prior graph in $\mA^{prior}$. As an example, consider three nodes with seven binary features $[0, 0, 0, 1, 1, 1, 1]$, $[0, 0, 0, 0, 0, 0, 0]$ and $[1, 1, 1, 1, 1, 1, 1]$ respectively and assume the last two features for each node are always equivalent and are computed based on a \emph{logical and} of the 4th and 5th features\footnote{For the first node in the example, the 4th and 5th features are both $1$ so their logical and is also $1$ and so the last two features for this node are both $1$. The computation for the other two nodes is similar.}. Without these two features, the first node is more similar to the second than the third node, but when considering these derived features, it becomes more similar to the third node. This change in node similarities affects the construction of $\mA^{prior}$ which can deteriorate the overall performance of the model. The version of SLAPS with the self-supervised task, on the other hand, is not affected by this problem as much because the model can learn to predict the derived node features based on other features and without heavily relying on the graph structure.

\item While many graph structures may be appropriate for predicting the node features, in the manual approach we only regularize toward one particular such structure. Using the self-supervised task, however, SLAPS can learn any of those structures; ideally, it learns the one that is more suited for the downstream task due to the extra supervision coming from the downstream task.
\end{itemize}

\section{Implementation Details}
We implemented our model in PyTorch \citep{paszke2017automatic}, used deep graph library (DGL) \citep{wang2019deep} for the sparse operations, and used Adam~\citep{adam} as the optimizer. We performed early stopping and hyperparameter tuning based on the accuracy on the validation set for all datasets except Wine and Cancer. For these two datasets, the validation accuracy reached 100 percent with many hyperparameter settings, making it difficult to select the best set of hyperparameters. Instead, we used the validation cross-entropy loss for these two datasets. 

We fixed the maximum number of epochs to $2000$.
We use two-layer GCNs for both $\func{GNN_C}$ and $\func{GNN_{DAE}}$ as well as for baselines and two-layer MLPs throughout the paper (for experiments on ogbn-arxiv, although the original paper uses models with three layers and with batch normalization after each layer, to be consistent with our other experiments we used two layers and removed the normalization). We used two learning rates, one for $\func{GCN_C}$ as $lr_{C}$ and one for the other parameters of the models as $lr_{DAE}$. We tuned the two learning rates from the set $\{0.01, 0.001\}$. We added dropout layers with dropout probabilities of $0.5$ after the first layer of the GNNs. We also added dropout to the adjacency matrix for both $\func{GNN_C}$ and $\func{GNN_{DAE}}$ as $dropout_{C}$ $dropout_{DAE}$ respectively and tuned the values from the set $\{0.25, 0.5\}$. 
We set the hidden dimension of $\func{GNN_C}$ to $32$ for all datasets except for ogbn-arxiv for which we set it to $256$.
We used cosine similarity for building the kNN graphs and tuned the value of $k$ from the set $\{10, 15, 20, 30\}$.
We tuned $\lambda$ ($\lambda$ controls the relative importance of the two losses) from the set $\{0.1, 1, 10, 100, 500\}$. We tuned $r$ and $\eta$ from the sets $\{1, 5, 10\}$ and $\{1, 5\}$ respectively.
The best set of hyperparameters for each dataset chosen on the validation set is in table~\ref{tab:hyperparameters}.
The code of our experiments will be available upon acceptance of the paper.

For GRCN~\citep{GRCN}, DGCNN~\citep{DGCNN}, and IDGL~\citep{IDGL}, we used the code released by the authors and tuned the hyperparameters as suggested in the original papers. 
The results of LDS~\citep{franceschi2019learning} are directly taken from the original paper. 
For LP~\cite{zhu2003semi}, we used scikit-learn python package~\cite{pedregosa2011scikit}.

All the results for our model and the baselines are averaged over 10 runs. We report the mean and standard deviation.
We ran all the experiments on a single GPU (NVIDIA GeForce GTX 1080 Ti).

\textbf{Self-training and AdaEdge:} We combined SLAPS (and kNN-GCN) with two techniques from the literature namely \emph{self-training} and \emph{AdaEdge}. For completeness sake, we provide a brief description of these approaches and refer the reader to the original papers for detailed descriptions.

For self-training, we first trained a model using the existing labels in the training set. Then we used this model to make predictions for the unlabeled nodes that were not in the train, validation, or test sets. We considered the label predictions for the top $\zeta$ most confident unlabeled nodes as ground truth labels and added them to the training labels. Finally, we trained a model from scratch on the expanded set of labels. Here, $\zeta$ is a hyperparameter. We tuned its value from the set $\{50, 100, 200, 300, 400, 500\}$.

For AdaEdge, in the case of kNN-GCN, we first trained a kNN-GCN model. Then we changed the structure of the graph from the kNN graph to a new graph by following these steps: 1) add edges between nodes with the same class predictions if both prediction confidences surpass a threshold, 2) remove edge between nodes with different class predictions if both prediction confidences surpass a threshold. Then, we trained a GCN model on the new structure and repeated the aforementioned steps to generate a new structure. We did this iteratively until generating a new structure did not provide a boost in performance on the validation set. For SLAPS, we followed a similar approach except that the initial model was a SLAPS model instead of a kNN-GCN model.

\textbf{kNN Implementation:} For our MLP generator, we used a kNN operation to sparsify the generated graph. Here, we explain how we implemented the kNN operation to avoid blocking the gradient flow. Let $\mM\in\mathbb{R}^{n\times n}$ with $\mM_{ij}=1$ if $v_j$ is among the top $k$ similar nodes to $v_i$ and $0$ otherwise, and let $\mS\in\mathbb{R}^{n\times n}$ with $\mS_{ij} = \func{Sim}(\mX'_i, \mX'_j)$ for some differentiable similarity function $\func{Sim}$ (we used cosine). Then  $\tilde{\mA}=\func{kNN}(\mX')=\mM\odot\mS$ where $\odot$ represents the Hadamard (element-wise) product. With this formulation, in the forward phase of the network, one can first compute the matrix $\mM$ using an off-the-shelf k-nearest neighbors algorithm and then compute the similarities in $\mS$ only for pairs of nodes where $\mM_{ij}=1$. In our experiments, we compute exact k-nearest neighbors; one can approximate it using locality-sensitive
hashing approaches for larger graphs (see, e.g., \cite{halcrow2020grale,kitaev2020reformer}). In the backward phase of our model, we compute the gradients only with respect to those elements in $\mS$ whose corresponding value in $\mM$ is $1$ (i.e. those elements $\mS_{ij}$ such that $\mM_{ij}=1$); the gradient with respect to the other elements is $0$.
Since $\mS$ is computed based on $\mX'$, the gradients flow to the elements in $\mX'$ (and consequently to the weights of the MLP) through $\mS$. 

\textbf{Adjacency processor:} We used a function $\func{P}$ in our adjacency processor to make the values of the $\tilde{\mA}$ positive. In our experiments, when using an MLP generator, we let $\func{P}$ be the ReLU function applied element-wise on the elements of $\tilde{\mA}$. When using the fully-parameterized (FP) generator, applying ReLU results in a gradient flow problem as any edge whose corresponding value in $\tilde{\mA}$ becomes less than or equal to zero stops receiving gradient updates. For this reason, for FP we apply the ELU~\cite{elu} function to the elements of $\tilde{\mA}$ and then add a value of $1$.

\section{Dataset statistics}
The statistics of the datasets used in the experiments can be found in Table~\ref{tab:datasets}.

\begin{table*}
\footnotesize
\caption{Best set of hyperparameters for different datasets chosen on validation set.}
\label{tab:hyperparameters}
\begin{center}
\begin{tabular}{c|c|cccccccccc}
Dataset & Generator& $lr_{C}$ & $lr_{DAE}$ & $dropout_{c}$ & $dropout_{DAE}$ & $k$ & $\lambda$ & $r$ & $\eta$\\ \hline
Cora & FP & 0.001 & 0.01 & 0.5 & 0.25& 30 & 10 & 10 & 5\\
Cora & MLP & 0.01 & 0.001 & 0.25 & 0.5 & 20 & 10 & 10 & 5\\
Cora & MLP-D & 0.01 & 0.001 & 0.25 & 0.5 & 15 & 10 & 10 & 5\\ \hline
Citeseer & FP & 0.01 & 0.01 & 0.5 & 0.5 & 30 & 1 & 10 & 1\\
Citeseer & MLP & 0.01 & 0.001 & 0.25 & 0.5 & 30 & 10 & 10 & 5\\
Citeseer & MLP-D & 0.001 & 0.01 & 0.5 & 0.5 & 20 & 10 & 10 & 5\\ \hline
Cora390 & FP & 0.01 & 0.01 & 0.25 & 0.5 & 20 & 100 & 10 & 5\\
Cora390 & MLP & 0.01 & 0.001 & 0.25 & 0.5 & 20 & 10 & 10 & 5\\
Cora390 & MLP-D & 0.001 & 0.001 & 0.25 & 0.5 & 20 & 10 & 10 & 5\\ \hline
Citeseer370 & FP & 0.01 & 0.01 & 0.5 & 0.5 & 30 & 1 & 10 & 1\\
Citeseer370 & MLP & 0.01 & 0.001 & 0.25 & 0.5 & 30 & 10 & 10 & 5\\
Citeseer370 & MLP-D & 0.01 & 0.01 & 0.25 & 0.5 & 20 & 10 & 10 & 5\\ \hline
Pubmed & MLP & 0.01 & 0.01 & 0.5 & 0.5 & 15 & 10 & 10 & 5\\
Pubmed & MLP-D & 0.01 & 0.01 & 0.25 & 0.25 & 15 & 100 & 5 & 5\\ \hline
ogbn-arxiv & MLP & 0.01 & 0.001 & 0.25 & 0.5 & 15 & 10 & 1 & 5\\
ogbn-arxiv & MLP-D & 0.01 & 0.001 & 0.5 & 0.25 & 15 & 10 & 1 & 5\\ \hline
Wine & FP & 0.01 & 0.001 & 0.5 & 0.5 & 20 & 0.1 & 5 & 5\\
Wine & MLP & 0.01 & 0.001 & 0.5 & 0.25 & 20 & 0.1 & 5 & 5\\
Wine & MLP-D & 0.01 & 0.01 & 0.25 & 0.5 & 10 & 1 & 5 & 5\\ \hline
Cancer & FP & 0.01 & 0.001 & 0.5 & 0.25 & 20 & 0.1 & 5 & 5\\
Cancer & MLP & 0.01 & 0.001 & 0.5 & 0.5 & 20 & 1.0 & 5 & 5\\
Cancer & MLP-D & 0.01 & 0.01 & 0.5 & 0.5 & 20 & 0.1 & 5 & 5\\ \hline
Digits & FP & 0.01 & 0.001 & 0.25 & 0.5 & 20 & 0.1 & 5 & 5\\
Digits & MLP & 0.01 & 0.001 & 0.25 & 0.5 & 20 & 10 & 5 & 5\\
Digits & MLP-D & 0.01 & 0.001 & 0.5 & 0.25 & 15 & 0.1 & 5 & 5\\ \hline
20news & FP & 0.01 & 0.01 & 0.5 & 0.5 & 20 & 500 & 5 & 5\\
20news & MLP & 0.001 & 0.001 & 0.25 & 0.5 & 20 & 500 & 5 & 5\\
20news & MLP-D & 0.01 & 0.01 & 0.25 & 0.25 & 20 & 100 & 5 & 5\\
\hline
MNIST (1000) & MLP & 0.01 & 0.01 & 0.5 & 0.5 & 15 & 10 & 10 & 5\\
MNIST (2000) & MLP-D & 0.01 & 0.001 & 0.5 & 0.5 & 15 & 100 & 10 & 5\\
MNIST (3000) & MLP & 0.01 & 0.01 & 0.5 & 0.5 & 15 & 10 & 5 & 5
\end{tabular}
\end{center}
\end{table*}

\section{Supervision starvation in Erd\H{o}s-R\'enyi and scale-free networks}

We start by defining some new notation that helps simplify the proofs and analysis in this section. We let $l_{v}$ be a random variable indicating that $v$ is a labeled node, with $\overline{l_{v}}$ indicating that its negation, $c_{v,u}$ be a random variable indicating that $v$ is connected to $u$ with an edge, with $\overline{c_{v,u}}$ indicating its negation, and $cl_v$ be random variable indicating that $v$ is connected to at least one labeled node with $\overline{cl_v}$ indicating its negation (i.e. it indicates that $v$ is connected to no labeled nodes).  

\begin{theorem}\label{th:main}
Let $\graph{G}(n, m)$ be an Erd\H{o}s-R\'enyi graph with $n$ nodes and $m$ edges. Assume we have labels for $q$ nodes selected uniformly at random. 
The probability of an edge being a starved edge with a two-layer GCN is equal to $(1 - \frac{q}{n})(1 - \frac{q}{n-1})\prod_{i=1}^{2q}(1 - \frac{m-1}{{n \choose 2}-i})$.
\end{theorem}

\begin{proof}
To compute the probability of an edge being a starved edge, we first compute the probability of the two nodes of the edge being unlabeled themselves and then the probability of the two nodes not being connected to any labeled nodes. Let $v$ and $u$ represent two nodes connected by an edge.

With $n$ nodes and $q$ labels, the probability of a node being labeled is $\frac{q}{n}$. Therefore, $Pr(\overline{l_v})=(1 - \frac{q}{n})$ and $Pr(\overline{l_u} \mid \overline{l_v})=(1 - \frac{q}{n-1})$. Therefore, $Pr(\overline{l_v}\wedge\overline{l_u})=(1 - \frac{q}{n})(1 - \frac{q}{n-1})$. 

Since there is an edge between $v$ and $v$, there are $m-1$ edges remaining. Also, there are ${n \choose 2} - 1$ pairs of nodes that can potentially have an edge between them. Therefore, the probability of $v$ being disconnected from the first labeled node is $1-\frac{m-1}{{n \choose 2}-1}$. If $v$ is disconnected from the first labeled node, there are still $m-1$ edges remaining and there are now ${n \choose 2} - 2$ pairs of nodes that can potentially have an edge between them. So the probability of $v$ being disconnected from the second node given that it is disconnected from the first labeled node is $1-\frac{m-1}{{n \choose 2}-2}$. With similar reasoning, we can see that the probability of $v$ being disconnected from the $i$-th labeled node given that it is disconnected from the first $i-1$ labeled nodes is $1-\frac{m-1}{{n \choose 2}-i}$.

We can follow similar reasoning for $u$. The probability of $u$ being disconnected from the first labeled node given that $v$ is disconnected from all $q$ labeled nodes is $1-\frac{m-1}{{n \choose 2}-q-1}$. That is because there are still $m-1$ edges remaining and ${n \choose 2}-q-1$ pairs of nodes that can potentially be connected with an edge. We can also see that the probability of $u$ being disconnected from the $i$-th labeled node given that it is disconnected from the first $i-1$ labeled nodes and that $v$ is disconnected from all $q$ labeled nodes is $1-\frac{m-1}{{n \choose 2}-q-i}$.

As the probability of the two nodes being unlabeled and not being connected to any labeled nodes in the graph are independent, their joint probability is the multiplication of their probabilities computed above and it is equal to $(1 - \frac{q}{n})(1 - \frac{q}{n-1})\prod_{i=1}^{2q}(1 - \frac{m-1}{{n \choose 2}-i})$.
\end{proof}

\textbf{Barabási–Albert and scale-free networks:} We also extend the above result for Erd\H{o}s-R\'enyi graphs to the Barabási–Albert~\cite{barabasi1999emergence} model. Since Barabási–Albert graph generation results in scale-free networks with a scale parameter $\gamma=-3$, we present results for the general case of scale-free networks as it makes the analysis simpler and more general. In what follows, we compute the probability of an edge being a starved edge in a scale-free network.

Let $\graph{G}$ be a scale-free network with $n$ nodes, $q$ labels (selected uniformly at random), and scale parameter $\gamma$. Then, if we select a random edge between two nodes $v$ and $u$, the probability of the edge between them being a starved edge is:

\begin{equation*}
Pr(\overline{l_v}) * Pr(\overline{l_u} | \overline{l_v}) * Pr(\overline{cl_v} | c_{v,u},\overline{l_v},\overline{l_u}) * Pr(\overline{cl_u} | c_{v,u},\overline{l_v},\overline{l_u},\overline{cl_v}).
\end{equation*}

Each of these terms can be computed as follows ($a \choose b$ represents the number of combinations of selecting $b$ items from a set with $a$ items):

\begin{itemize}
    \item $Pr(\overline{l_v}) = (1 - \frac{q}{n})$
    \item $Pr(\overline{l_u} | \overline{l_v}) = (1 - \frac{q}{n-1})$
    \item $Pr(\overline{cl_v} | c_{v,u},\overline{l_u},\overline{l_v}) = \frac{\sum^{n-1}_{k=1} k^{\gamma} \frac{\binom{n - q - 2}{k-1}}{\binom{n-2}{k-1}}}{\sum^{n-1}_{k=1}k^{\gamma}}$
\end{itemize}

For a large enough network, $Pr(\overline{cl_u} | c_{v,u},\overline{l_v},\overline{l_u},\overline{cl_v})$
can be approximated as  $Pr(\overline{cl_u} | c_{v,u},\overline{l_v},\overline{l_u})$ and it can be computed similarly as the previous case.

With the derivation above, for a scale-free network with $n=2708$ and $q=140$ (corresponding to the stats from Cora), the probability of an edge being a starved edge for $\gamma=-3$ is $0.87$ and for $\gamma=-2$ is $0.76$ .

\begin{table*}
\caption{Dataset statistics.}
\label{tab:datasets}
\begin{center}
\begin{tabular}{c|cccccc}
Dataset & Nodes & Edges & Classes & Features & Label rate \\ \hline
Cora & 2,708 & 5,429 & 7 & 1,433 & 0.052 \\
Citeseer & 3,327 & 4,732 & 6 & 3,703 & 0.036\\
Pubmed & 19,717 & 44,338 & 3 & 500 & 0.003\\
ogbn-arxiv & 169,343 & 1,166,243 & 40 & 128 & 0.537 \\
Wine & 178 & 0 & 3 & 13 & 0.112\\
Cancer & 569 & 0 & 2 & 30 & 0.035\\
Digits & 1,797 & 0 & 10 & 64 & 0.056\\
20news & 9,607 & 0 & 10 & 236 & 0.021\\
MNIST & 10,000 & 0 & 10 & 784 & 0.1, 0.2 and 0.3 \\
\end{tabular}
\end{center}
\end{table*}

\section{Why not compare the learned graph structures with the original ones?}
A comparison between the learned graph structures using SLAPS (or other baselines) and the original graph structure of the datasets we used may not be sensible. We explain this using an example. Before getting into the example, 
we remind the reader that the goal of structure learning for semi-supervised classification with graph neural networks is to learn a structure with a high degree of homophily. Following \cite{zhu2020beyond}, we define the \emph{edge homophily ratio} as the fraction of edges in the
graph that connect nodes that have the same class label. 

Figure~\ref{fig:homophily} demonstrates an example where two graph structures for the same set of nodes have the same edge homophily ratio (0.8 for both) but have no edges in common. For our task, it is possible that the original graph structure (e.g., the citation graph in Cora) corresponds to the structure on the left but SLAPS (or any other model) learns the graph on the right, or vice versa. While both these structures may be equally good\footnote{We are disregarding the features for simplicity sake.}, they do not share any edges. Therefore, measuring the quality of the learned graph using SLAPS by comparing it to the original graph of the datasets may not be sensible. However, if a noisy version of the initial structure is provided as input for SLAPS, then one may expect that SLAPS recovers a structure similar to the cleaned original graph and this is indeed what we demonstrate in the main text.

\begin{figure}[t]
   \includegraphics[width=\columnwidth]{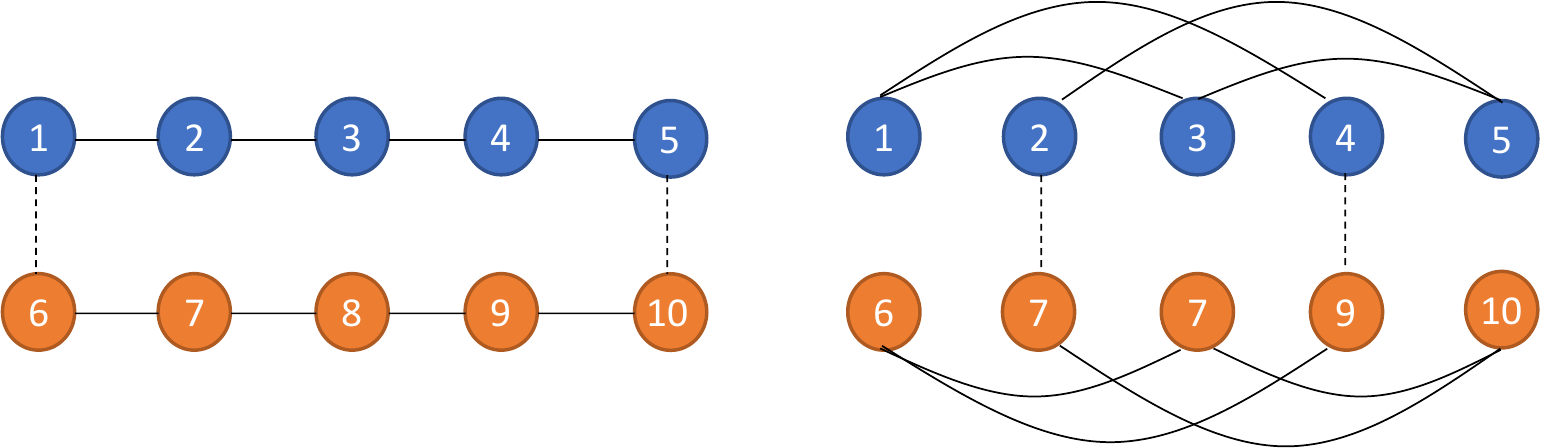}
   \caption{%
   \label{fig:homophily} %
   Two example graph structures. Node colors indicates the class labels. Solid lines indicate homophilous edges and dashed lines indicate non-homophilous edges. The two graphs exhibit the same degree of homophily yet there is not overlap between their edges.
   }
\end{figure}

\section{Limitations}
In this section, we discuss some of the limitations of the proposed model. Firstly, in cases where nodes do not have input features but an initial noisy structure of the nodes is available, our self-supervised task cannot be readily applied. One possible solution is to first run an unsupervised node embedding model such as DeepWalk \cite{perozzi2014deepwalk} to obtain node embeddings, then treat these embeddings as node features and run SLAPS. 
Secondly, the FP graph generator is not applicable in the inductive setting; this is because FP directly optimizes the adjacency matrix. However, our other two graph generators (MLP and MLP-D) can be applied in the inductive setting.

\end{document}

%% file: math_commands.tex
\usepackage{amsmath,amsfonts,bm}

\def\eqref#1{equation~\ref{#1}}
\def\Eqref#1{Equation~\ref{#1}}

\def\1{\bm{1}}

\def\vv{{\bm{v}}}

\def\mA{{\bm{A}}}

\def\mD{{\bm{D}}}

\def\mH{{\bm{H}}}
\def\mI{{\bm{I}}}

\def\mM{{\bm{M}}}

\def\mS{{\bm{S}}}

\def\mW{{\bm{W}}}
\def\mX{{\bm{X}}}

\DeclareMathAlphabet{\mathsfit}{\encodingdefault}{\sfdefault}{m}{sl}
\SetMathAlphabet{\mathsfit}{bold}{\encodingdefault}{\sfdefault}{bx}{n}

